\documentclass{article}

     \PassOptionsToPackage{numbers, compress}{natbib}

     \usepackage[preprint]{neurips_2019}

\usepackage[utf8]{inputenc} 
\usepackage[T1]{fontenc}    
\usepackage[colorlinks=true, allcolors=blue]{hyperref}       
\usepackage{url}            
\usepackage{booktabs}       
\usepackage{amsfonts}       
\usepackage{nicefrac}       
\usepackage{microtype}      

\usepackage{amsfonts, amsmath, amssymb, amsthm}
\usepackage{xfrac}
\usepackage{graphicx, color}
\usepackage{natbib}
\usepackage{dsfont}
\usepackage{bbm}
\usepackage{algorithm}
\usepackage{import}
\usepackage{enumitem}
\usepackage{mathtools}
\usepackage{svg}
\usepackage{siunitx}

\newtheorem{lemma}{Lemma}
\newtheorem{theorem}{Theorem}

\newtheorem{proposition}{Proposition}
\theoremstyle{definition}

\newtheorem{assumption}{Assumption}

\DeclareMathOperator{\kl}{kl}

\DeclareMathOperator{\EE}{\mathbb{E}}
\DeclareMathOperator{\PP}{\mathbb{P}}
\DeclareMathOperator{\R}{\mathbb{R}}
\DeclareMathOperator{\N}{\mathbb{N}}

\DeclareMathOperator{\Ng}{\mathcal{N}}
\DeclareMathOperator*{\argmax}{arg\,max}
\DeclareMathOperator*{\argmin}{arg\,min}
\DeclareMathOperator{\clip}{\mathrm{Clip}}

\DeclareMathOperator{\loglog}{\log\!\log}

\newcommand*\diff{\mathop{}\!\mathrm{d}}
\renewcommand{\d}{\mathrm{d}}
\newcommand{\hmu}{\widehat{\mu}}

\newcommand{\hi}{\widehat{i}}
\renewcommand{\epsilon}{\varepsilon}

\newcommand{\wstar}{w^\star}

\newcommand{\bmu}{\bar{\mu}}
\newcommand{\ind}{\mathds{1}}
\newcommand{\cS}{\mathcal{S}}
\newcommand{\cM}{\mathcal{M}}
\newcommand{\cI}{\mathcal{I}}
\newcommand{\cF}{\mathcal{F}}
\newcommand{\cB}{\mathcal{B}}

\newcommand{\cE}{\mathcal{E}}
\newcommand{\cP}{\mathcal{P}}

\DeclarePairedDelimiter\floor{\lfloor}{\rfloor}

\newcommand{\tw}{\widetilde{w}}

\newcommand{\thr}{\mathfrak{T}}
\newcommand{\alt}{\mathcal{A}\textit{lt}}

\newcommand{\htheta}{\widehat{\theta}}
\newcommand{\tf}{\widetilde{f}}
\newcommand{\tT}{\widetilde{T}}

\title{Gradient Ascent for Active Exploration in Bandit Problems}

\author{%
  Pierre M\'enard \\
  Institut de Mathématiques de Toulouse, Université de Toulouse, Toulouse\\
  IRT Saint Exupéry, Toulouse\\
  \texttt{pierre.menard@univ-toulouse.fr} \\
}

\begin{document}

\maketitle

\begin{abstract}
We present a new algorithm based on an gradient ascent for a general Active Exploration bandit problem in the fixed confidence setting. This problem encompasses several well studied problems such that the Best Arm Identification or Thresholding Bandits. It consists of a new sampling rule based on an online lazy mirror ascent. We prove that this algorithm is asymptotically optimal and, most importantly, computationally efficient.
\end{abstract}

\section{Introduction}
Several recent and less recent analyses of bandit problems share the remarkable feature that an instance-dependant lower-bound analysis permits to show the existence of an \emph{optimal proportion of draws}, which every efficient strategy needs to match, and which is used as a basis for the design of optimal algorithms. This is the case in Active Exploration bandit problems, see \citet{chernoff1959sequential}, \citet{soare2014best}, \citet{russo2016simple} and \citet{garivier2016optimal} but also for the Regret Minimization bandit problems, from the simplest multi-armed bandit setting \citet{garivier2018explore} to more complex setting \citet{lattimore2017end}, \citet{combes2017minimal}. To reach the asymptotic lower bounds one needs to sample asymptotically according to this optimal proportion of draws. A natural strategy is to sample according to the optimal proportion of draws associated with the current estimate of the true parameter, with some extra exploration. See for example \citet{antos2008active}, \citet{garivier2016optimal}, \citet{lattimore2017end} and \citet{combes2017minimal}. This strategy has a major drawback, computing the optimal proportion of draws requires to solve an often involved \emph{concave optimization problem}. Thus, this can lead to rather computationally inefficient strategy since one must solve exactly at each steps a new concave optimization problem. 

In this paper we propose to use instead a gradient ascent to solve in an online fashion the optimization problem thus merging the Active Exploration problem and the computation of the optimal proportion of draws. Precisely we perform an online lazy mirror ascent, see \citet{shalev2012online}, \citet{bubeck2011introduction}, adding an new link between stochastic bandits and online convex optimization. Hence, it is sufficient to compute at each steps only a (sub-)gradient, which greatly improves the computational complexity. As a byproduct the obtained algorithm is quite generic and can be applied in various Active Exploration bandit problems, see Appendix~\ref{app:examples}.

The paper is organized as follows. In Section~\ref{sec:problem_description} we define the framework. A general asymptotic lower bound is presented in Section~\ref{sec:lower_bound} . In Section~\ref{sec:intuition} we motivate the introduction of the gradient ascent. The main result, namely the asymptotic optimality of Algorithm~\ref{alg:gradient_ascent} and its proof compose Section~\ref{sec:gradient_ascent}. Section~\ref{app:examples} regroups various examples that are described by the general setting introduced in Section~\ref{sec:problem_description}. Section~\ref{sec:experiments} reports results of some numerical experiments comparing Algorithm~\ref{alg:gradient_ascent} to its competitors.
\label{sec:intro}


\paragraph{Notation.} For $K\in\N^*$, let $[1,K]=\{1,\ldots,K\}$ be the set of integers lower than or equal to $K$. We denote by $\Sigma_K$ the simplex of dimension $K-1$ and by $\{e_a\}_{a\in [1,K]}$ the canonical basis of $\R^K$.  A distribution on $[1,K]$ is assimilated to an element of $\Sigma_K$. The Kullback-Leibler divergence between two probability distributions $w,w'$ on $[1,K]$ is (with the usual conventions)
\[
\kl(w,w')=\sum_{a=1}^K w_a \log\!\!\left(\frac{w_a}{w'_a}\right)\,.
\]
\subsection{Problem description}
\label{sec:problem_description}
For $K\geq 2$, we consider a Gaussian bandit problem $\big( \Ng(\mu_1,\sigma^2),\ldots,\Ng(\mu_K,\sigma^2)\big)$, which we unambiguously refer to by the vector of means $ \mu=\big(\mu_1,\ldots,\mu_K\big)$. Without loss of generality, we set in the following $\sigma^2=1$. We denote by $\cM$ the set of  Gaussian bandit problems. Let $\PP_{\mu}$ and $\EE_{\mu}$ be respectively the probability and the expectation under the bandit problem $ \mu$.

We fix a finite number of subsets of bandit problems $\cS_i \subset \cM$ for $i \in \cI$ with $|\cI|<\infty$ and we assume that the subsets $\cS_i$ are pairwise disjoint, open and convex. We will explain latter why we need these assumptions on the sets $\cS_i$. For a certain bandit problem $\mu$ in $\cS:=\cup_{i\in\cI}\cS_i$ our objective is to identify to which set it belongs, i.e. to \emph{find $i(\mu)$ such that $\mu \in S_{i(\mu)}$}. Namely, we consider algorithms that output a subset index $\hi \in \cI$ after $\tau>0$ pulls. This setting is quite general and encompasses several Active Exploration bandit problems, see Section~\ref{app:examples}.

Two approaches for this problem have been proposed: first, one may consider a given budget $\tau$ and try to minimize the probability to predict a wrong subset index, this is the \emph{Fixed Budget setting}, see \citet{bubeck2012regret}, \citet{audibert2010best} and \citet{LocatelliGC16}. The second approach is the \emph{Fixed Confidence setting}, where we fix a confidence level $\delta$ and try to minimize the expected number of sample $\EE_\mu[\tau_\delta]$ under the constraint that the predicted subset index is the right one with probability at least $1-\delta$, see \citet{chernoff1959sequential}, \citet{even2002pac}, \citet{mannor2004sample} and \citet{KaCaGa16}. In this paper we will consider the second approach. 

The game goes as follow: at each round $t\in\N^*$ the agent chooses an arm $A_t \in \{1,\ldots,K\}$ and observes a sample $Y_t\sim\Ng(\mu_{A_t},1)$ conditionally independent from the past. Let $\cF_{t}=\sigma (A_1,Y_1,\ldots, A_t,Y_t)$ be the information available to the agent at time $t$. In order to respect the confidence constraint the agent must follow a \emph{$\delta$-correct} algorithm comprised of:
\begin{itemize}[noitemsep,nolistsep]
\item[-] a \emph{sampling rule} $(A_t)_{t\geq 1}$, where $A_t$ is $\cF_{t-1}$-measurable,
\item[-] a \emph{stopping rule} $\tau_\delta$, a stopping time for the filtration $(\cF_t)_{t\geq 1}$,
\item[-] a \emph{decision rule} $\hi$ $\cF_{\tau_\delta}$-measurable,
\end{itemize}
such that for all $\mu\in\cS$ the fixed confidence condition is satisfied $\PP_{\mu}\big(\hi\neq i(\mu)\big)\leq \delta$ and that the algorithm stop almost surely $\PP_{\mu}\big(\tau_\delta < \infty)= 1$. In this paper we will focus our attention on the \emph{sampling rule} since stopping rules are now well understood and decision rule are straightforward to find.

\subsection{Lower Bound}
\label{sec:lower_bound}
The Kullback-Leibler divergence between two Gaussian distributions $\Ng(\mu_1,1)$ and $\Ng(\mu_2,1)$ is defined by
\[
\d(\mu_1,\mu_2):=\frac{(\mu_1-\mu_2)^2}{2}\,.
\]
The set of alternatives of the problem $\mu\in\cS$ is denoted by $\alt(\mu):=\cup_{i\neq i(\mu)}\cS_i$. One can prove the following generic asymptotic lower bound on the expected number of samples when the confidence level $\delta$ tends to zero, see \citet{garivier2016optimal} and \citet{garivier2017thresholding}.
\begin{theorem}
\label{th:lb_asympt_threshold}
For all $\mu\in\cS$, for all $0<\delta<1/2$,
\begin{equation}
    \EE_{\mu}[\tau_\delta]\geq T^\star (\mu) \kl(\delta,1-\delta)\,,
    \label{eq:LB_asymp}
\end{equation}
where the characteristic time $T^\star (\mu)$ is defined by 
\begin{equation}
   T^\star (\mu)^{-1}=\max_{w\in\Sigma_K}\inf_{\lambda \in \alt(\mu)} \sum_{a=1}^{K} w_a \d(\mu_a, \lambda_a)\,.
   \label{eq:charateristic_time}
\end{equation}
In particular \eqref{eq:LB_asymp} implies that
\begin{equation}
    \liminf\limits_{\delta\rightarrow 0} \frac{\EE_{\mu}[\tau_\delta]}{\log(1/\delta)}\geq T^\star(\mu)\,.
    \label{eq:LB_asymp_lim}
\end{equation}
\end{theorem}
As already explained by~\citet{chernoff1959sequential}, it is interesting to note that asymptotically we end up with a zero-sum game where the agent first plays a proportion of draws $w$ trying to minimize the sum in \eqref{eq:charateristic_time} then the "nature" plays an alternative $\lambda$ trying to do the opposite. The value of this game is exactly $T^\star(\mu)^{-1}$.  In the sequel we denote by  
\begin{equation}
\label{eq:def_F}
F(w,\mu):= \inf_{\lambda \in \alt(\mu)} \sum_{a=1}^{K} w_a \d(\mu_a, \lambda_a)\,,
\end{equation}
the function that the agent needs to maximize against a "nature" that plays optimally.
An algorithm is thus asymptotically optimal if the reverse inequality of \eqref{eq:LB_asymp_lim} holds with a limsup instead of a liminf.
\subsection{Intuition: what is the idea behind the algorithm?}
\label{sec:intuition}
To get an asymptotically optimal algorithm the agent wants to play accordingly to an optimal proportion of draws $\wstar(\mu)$, defined by 
\begin{equation}
\wstar(\mu)\in \argmax_{w\in\Sigma_K}\inf_{\lambda \in \alt(\mu)} \sum_{a=1}^{K} w_a \d(\mu_a, \lambda_a)\,,
\label{eq:def_w_star}
\end{equation}
in order to minimize the characteristic time in~\eqref{eq:charateristic_time}. But, of course, the agent has not access to the true vector of means. One way to settle this problem is to track the optimal proportion of the current empirical means. Let $\hmu(t)$ be the vector of empirical means at time $t$:
\begin{equation*}
\hmu_a(t)=\frac{1}{N_a(t)}\sum_{s=1}^{t} Y_s\, \ind_{\{A_s=a\}}\,,
\end{equation*}
where $N_a(t) = \sum_{s=1}^t \ind_{\{ A_s = a \}}$ denotes the number of draws of arm $a$ up to and including time $t$. We will denote by $w(t) = N(t)/t$ the empirical proportion of draws at time $t$. Following this idea, the sampling rule could be 
\[
A_{t+1}\in\argmax_{a\in [1,K]} \wstar_a\big(\hmu(t)\big) -w_a(t)\,.
\]
This rule is equivalent to the direct tracking rule (without forced exploration, see below) by \citet{garivier2016optimal}. But this approach has a major drawback, at each time $t$ we need to solve exactly the concave optimization problem in~\eqref{eq:def_w_star}. And it appears that in some case we can not solve it analytically, see for example \citet{garivier2017thresholding}. Even if there exists an efficient way to solve the optimization problem numerically like for example in the Best Arm Identification problem some simplest and efficient algorithms give experimentally comparable results. We can cite for example Best Challenger type algorithms, see \citet{garivier2016optimal} and \citet{russo2016simple}. 

The idea of our algorithm is best explained on the simple example of the Thresholding Bandit problem (see Section~\ref{sec:thresholding_bandit}), where the set of all arms larger than the threshold $\thr$ is to be identified. There exists a natural and efficient sampling rule (see \citet{LocatelliGC16}): 
\begin{equation}
    A_{t+1}\in\argmin_{a\in[1,K]} N_a(t) \d\big(\hmu_a(t),\thr\big)\,.
\label{eq:sampling_rule_threshold}
\end{equation}It turns out that this sampling rule leads to an asymptotically optimal algorithm. We are not aware of a reference for this fact. In order to give an interpretation of this sampling rule, let takes one step back. In this problem we want to maximize with respect to the first variable the following concave function (see Section~\ref{sec:thresholding_bandit})
\begin{equation}
\label{eq:def_F_thresholding}
F(w,\mu) = \min_{a\in[1,K]} w_a \d(\mu_a, \thr)\,.
\end{equation}
The sub-gradient of $F(\cdot,\mu)$ at $w$, denoted by $\partial F(w,\mu)$, is a convex combination of the vectors 
\[
\nabla F(w,\mu)=\begin{bmatrix}
(0)\\
\d(\mu_b, \thr)\\
(0)
\end{bmatrix}\!,
\]
for the active coordinates $b$ that attain the minimum in \eqref{eq:def_F_thresholding}. With this notation, the sampling rule~\eqref{eq:sampling_rule_threshold} can be rewritten in the following form 
\[
e_{A_{t+1}}\in\argmax_{w\in\Sigma_K} w\cdot \nabla F\big(w(t),\hmu(t)\big)\,,
\]
where $\nabla F\big(w(t),\hmu(t)\big)$ is some element in the sub-gradient $\partial F\big(w(t),\hmu(t)\big)$. Then the update of the empirical proportion of draws follows the simple rule 
\begin{equation}
\label{eq:update_franck_wolf}
w(t+1)= \frac{t}{t+1} w(t) + \frac{1}{t+1} e_{A_{t+1}}\,.
\end{equation}
Here we recognize surprisingly one step of the Frank-Wolfe algorithm \citep{frank1956algorithm} for maximizing the concave function $F\big(\cdot,\hmu(t)\big)$ on the simplex. The exact same analysis can be done with a variant of the Best Challenger sampling rule for the Best Arm Identification problem. This is described in Section~\ref{sec:BAI}. It is not the first time that Frank-Wolfe algorithm appears in the stochastic bandits field, see for example \citet{berthet2017fast}. Precisely in the aforementioned reference they interpret the classical UCB algorithm as an instance of this algorithm with an "optimistic" gradient. The main difficulty here, which does not appear in the Regret Minimization problem, is that the function $F(\cdot,\mu)$ \emph{is not smooth} in general (as an infimum of linear functions). Thus we can not directly leverage the analysis of Frank-Wolfe algorithm in our setting as \citet{berthet2017fast}. In particular it is not obvious that the sampling rule driven by the Frank-Wolfe algorithm will converge to the maximum of $F(\cdot,\mu)$, for the general problem presented in Section~\ref{sec:intro}, even in the absence of noise (i.e. $\sigma=0$).

But we can keep the idea of using a concave optimizer in an online fashion instead of computing at each steps the optimal proportion of draws. Indeed there is a candidate of choice for optimizing non-smooth concave function namely the \emph{sub-gradient ascent}. Now the strategy is clear, at each steps we will perform one step of sub-gradient ascent for the function $F\big(\cdot,\hmu(t)\big)$ on the simplex. Nevertheless, the update of the proportion of draws will be more intricate than in \eqref{eq:update_franck_wolf}, we will need to track the average of weights proposed by the sub-gradient ascent and force some exploration, see next section for details. Note that this greatly improve the computational complexity of the algorithm since one just needs to compute an element of the sub-gradient of $F$ at each time step. In various setting this computation is straightforward, see Appendix~\ref{app:examples}, in general it boils down to compute the projection of the vector of empirical means on the closure of alternative sets thanks to the particular form of the function $F$, see~\eqref{eq:def_F}. Since the set $S_i$ are convex, if the weights $w(t)$ are strictly positive (which will be the case in Algorithm~\ref{alg:gradient_ascent}) the projection always exists.

\section{Gradient Ascent}
\label{sec:gradient_ascent}
Before presenting the algorithm we need to fix some notations. 
Since $\hmu(t)$ does not necessary lie in the set $\cS$, we first extend $F(w,\cdot)$ on the entire set $\cM$, by setting 
\begin{equation*}
\alt(\mu) = \begin{cases}
\cS \text{ if }\mu \notin \cS\\
\bigcup_{i\neq i(\mu)} \cS_{i} \text{ else}
\end{cases}
\,.
\end{equation*} 
Then, $\nabla F(w,\mu)$ will denote some element of the sub-gradient $\partial F(w,\mu)$ of $F(\cdot,\mu)$ at $w$.

As motivated in Section~\ref{sec:intuition}, we will perform a gradient ascent on the concave function $F\big(\cdot,\hmu(t)\big)$ to drive the sampling rule. More precisely we use an online lazy mirror ascent (see \citet{bubeck2015convex}) on the simplex, using the Kullback-Leibler divergence to the uniform distribution $\pi$ as mirror map:
\begin{align*}
\tw(t+1) = \argmax_{w\in\Sigma_K} \eta_{t+1} \sum_{s=K}^{t} w\cdot \clip_s\!\Big(\nabla F\big(\tw(s),\hmu(s)\big)\!\Big)-\kl(w,\pi) \,,
\end{align*}
where, for an arbitrary constant $M>0$, we clipped the gradient $\clip_t(x)=[\min(x_a,M\sqrt{t})]_{a\in[1,K]}$. This is just a technical trick to handle the fact that the gradient may be not uniformly bounded in the very first steps. In practice, however, this technical trick seems useless and we recommend to ignore it (that is, take $M=+\infty$). There is a closed formula for the weights $\tw(t+1)$, see Appendix~\ref{app:proof_online_regret}. Note that it is crucial here to use an anytime optimizer since we do not know in advance when the  algorithm will stop. Then we skew the weights $\tw(t)$ toward the uniform distribution $\pi$ to force exploration
\[
 w'(t+1)=(1-\gamma_t) \tw(t+1)+ \gamma_t \pi\,.
\]
This trick is quite usual as for example in the EXP3.P algorithm, see \citet{bubeck2012regret}. In some particular settings this extra exploration is not necessary, for example in the Thresholding Bandits problem. We believe that there is a more intrinsic way to perform exploration but this is out of the scope of this paper. Since we perform step size of order $\eta_t\sim 1/\sqrt{t}$ we can not use the same simple update rule of the empirical proportion of draws as in \eqref{eq:update_franck_wolf} where the steps size is of order $1/t$. But we can track the cumulative sum of weights $w'$ as follows
\[
A_{t+1}\in \argmax_{a\in[1,K]} \sum_{s=1}^{t+1} w'_a(s)- N_a(t)\,.
\]
It is important to track the cumulative sum of weights here because the analysis of the online mirror ascent provides only guarantees on the \emph{cumulative regret}.

For the stopping rule we use the classical Chernoff stopping rule~\eqref{eq:stopping_rule}, see \citet{chernoff1959sequential}, \citet{garivier2016optimal}, \citet{garivier2017thresholding},

That is, we stop when the vector of empirical means is far enough from any alternative with respect to the empirical Kullback-Leibler divergence. Note that, here, the threshold $\beta(N(t),\delta)$ does not depend directly on $t$, but via the vector of counts $N(t)$. This allows to use the maximal inequality of Proposition~\ref{prop:max_ineq}, which yields a very short and direct proof of $\delta$-correctness: see Section~\ref{sec:delta_correctness}.

The decision rule~\eqref{eq:decision_rule} just chooses the closest set $\cS_i$ to the vector of empirical means with respect to the empirical Kullback-Leibler divergence. Putting all together, we end up with Algorithm~\ref{alg:gradient_ascent}.

\begin{algorithm}[ht]
	\smallskip
	\textbf{Initialization} Pull each arms once and set $\tw(t)=w'(t)=\pi$ for all $1\leq t \leq K$\\
	\smallskip
	\textbf{Sampling rule}, for $t\geq K$\\
	Update the weights (sub-gradient ascent)
	\begin{equation}
\tw(t+1) = \argmax_{w\in\Sigma_K} \eta_{t+1} \sum_{s=K}^{t} w\cdot \clip_s\!\left(\nabla F\big(\tw(s),\hmu(s)\big)\right)-\kl(w,\pi) \,,
	    \label{eq:sampling_rule_gradient_ascent}
 	\end{equation}
	\begin{equation}
	    w'(t+1)=(1-\gamma_t) \tw(t+1)+ \gamma_t \pi\,.
	    \label{eq:sampling_rule_forced_exploration}
	\end{equation}
	Pull the arm (track the cumulative sum of weights)
	\begin{equation}
	    A_{t+1}\in \argmax_{a\in[1,K]} \sum_{s=1}^{t+1} w'_a(s)- N_a(t)\,. 
	    \label{eq:sampling_rule_tracking}
	\end{equation}
	\textbf{Stopping rule}\\
	    \begin{equation}
\displaystyle\tau_\delta=\inf\Big\{ t\geq K:\, \inf_{\lambda \in \alt(\hmu(t))} \sum_{a=1}^{K} N_a(t) \d\big(\hmu_a(t), \lambda_a\big)\geq \beta(N(t),\delta)
	        \Big\}\,.
	   \label{eq:stopping_rule}     
	    \end{equation}
	\textbf{Decision rule}\\
	\begin{equation}
	\label{eq:decision_rule}
	    \hi\in\argmin_{i\in\cI}\inf_{\lambda\in \cS_i}\sum_{a=1}^{K} N_a(\tau_\delta) \d(\hmu_a(\tau_\delta), \lambda_a)\,.
	\end{equation}
\caption{}	
	\label{alg:gradient_ascent}
\end{algorithm}

In order to preform a gradient descent we need that the sub-gradient of $F(\cdot,\mu)$ is bounded in a neighborhood of $\mu$. For the examples presented in Appendix~\ref{app:examples} or if the $\cS_i$ are bounded this assertion holds but for some pathological examples this assertion can be wrong (see Appendix~\ref{app:counter_example}). That why we make the following assumption where  we denote by $\cB_\infty(x,\kappa)$ the ball of radius $\kappa$ for the infinity norm $|\cdot|_\infty$ centered at $x$.
\begin{assumption}
\label{assp:bounded_gradient}
We assume that for all $\mu\in \cS$ there exists $\kappa_0$ that may depend on $\mu$ such that:
\[
\forall w\in \Sigma_K,\ \forall\mu'\in \cB_{\infty}(\mu,\kappa_0),\  \forall a\in[1,K],\,\qquad 0 \leq \nabla_a F(w,\mu')\leq L\,.
\]
\end{assumption}
We can now state the main result of the paper.
\begin{theorem}
For $\beta(N(t),\delta)$ given by \eqref{eq:def_beta}, $\eta_t=1/\sqrt{t}$, $\gamma_t=1/(4\sqrt{t})$, Algorithm~\ref{alg:gradient_ascent} is $\delta$-correct and asymptotically optimal, i.e.
\[
\limsup\limits_{\delta\rightarrow 0} \frac{\EE_{\mu}[\tau_\delta]}{\log(1/\delta)}\leq  T^\star(\mu)\,.
\]
\label{th:asymptotic_optimality}
\end{theorem}
In the rest of this section we will present the main lines of the proof of Theorem~\ref{th:asymptotic_optimality}. A detailed proof can be found in Appendix~\ref{app:proof_main_result}.

\subsection{$\delta$-correctness of Algorithm~\ref{alg:gradient_ascent}}
\label{sec:delta_correctness}
The $\delta$-correctness of Algorithm~\ref{alg:gradient_ascent} is a simple consequence of the following maximal inequality, see Appendix~\ref{app:deviations} for a proof.
\begin{proposition}[Maximal inequality]
For $\delta>0$ and the choice of the threshold
\begin{align}
    \beta\big(N(t),\delta\big)=\log(1/\delta)+K\log\!\big(4\log(1/\delta)+1\big)+6\sum_{a=1}^K \log\!\Big(\log\!\big(N_{a}(t)\big)+3\Big)+K\widetilde{C}\label{eq:def_beta}
\end{align}
where $\widetilde{C}$ is a universal constant defined in the proof of Proposition~\ref{prop:max_ineq_diag} in Appendix~\ref{app:deviations}, it holds
 \begin{equation}
    \label{eq:maximal_inequality_chernoff}
     \PP_\mu\!\!\left(\exists t\geq K,\, \sum_{a=1}^K  N_a(t) \d(\hmu_a(t), \mu_a) \geq \beta\big(N(t),\delta\big)\right)\leq \delta\,.
 \end{equation}
\label{prop:max_ineq}
\end{proposition}
Indeed, if the algorithm returns the wrong index $\hi\neq i(\mu)$ we know that the true parameter is in the set of alternatives at time $\tau$, i.e. $\mu\in\alt\big(\hmu(\tau_\delta)\big)$. Therefore thanks to the stopping rule \eqref{eq:stopping_rule} then  \eqref{eq:maximal_inequality_chernoff} it holds
\begin{align*}
    \PP_\mu\!\big(\hi\neq i(\mu)\big)\leq \PP_\mu\!\Bigg( \sum_{a=1}^{K} N_a(\tau_\delta) \d(\hmu_a(\tau_\delta), \mu) \geq \beta(N(\tau_\delta),\delta\big)\!\!\Bigg)\leq \delta\,.
\end{align*}
We will prove that $\tau_\delta$ is finite almost surely in the next section. 

\subsection{Asymptotic Optimality of Algorithm~\ref{alg:gradient_ascent}}
\label{sec:gradient_ascent_proof}
First we need some properties of regularity of the function $F$ around $\mu$ in order to prove a regret bound on the online lazy mirror ascent. In Appendix~\ref{app:other_proofs} we derive the following proposition.

\begin{proposition}[Regularity]
For all $\mu\in\cS$  and $\epsilon>0$ there exists constants $\kappa_\epsilon\leq \kappa_0,L >0$ that may depend on $\mu$ such that $\cB_{\infty}(\mu,\kappa_\epsilon)\subset\cS_{i(\mu)}\,,$ and $\forall \mu'\mu''\in\cB_{\infty}(\mu,\kappa_\epsilon),\,\forall w\in \Sigma_K$ it holds
\[
|\mu'-\mu''|_{\infty}\leq \kappa_\epsilon \Rightarrow |F(w,\mu')-F(w,\mu'')|\leq \epsilon\,.
\]
\label{prop:regularity}
\end{proposition}

Fix $\epsilon>0$ some real number and consider the typical event
\begin{equation*}
    \cE_\epsilon(T)=\bigcap_{t\geq g(T)}^T\big\{\hmu(t)\in\cB_{\infty}(\mu,\kappa_\epsilon)\big\}\,.
\end{equation*}
where $g(T)\sim T^{1/4}$, for some horizon $T$. We want to prove that for $T$ large enough, on the event $\cE_\epsilon(T)$, the difference between the maximum of $F$ for the true parameter, namely $T^\star(\mu)^{-1}=F\big(\wstar(\mu),\mu\big)$ and its empirical counterpart at time $T$, $F\big(w(T),\hmu(T)\big)$ is small, precisely of order $\epsilon$. To this aim we will use the following regret bound for the online lazy mirror ascent proved in Appendix~\ref{app:proof_online_regret}.
\begin{proposition}[Regret bound for the online lazy mirror ascent]
\label{prop:regret_bound}
For the weights $\tw(t)$ given by \eqref{eq:sampling_rule_gradient_ascent}, and a constant $C_0$ that depends on $K, L,M$, on the event $\cE_\epsilon(T)$ it holds
\begin{equation}
    \sum_{t=g(T)}^T F\big(\wstar(\mu),\hmu(t)\big)-F\big(\tw(t),\hmu(t)\big)\leq C_0\sqrt{T}\,.
    \label{eq:regret_bound}
\end{equation}
The expression of $C_0$ can be found in Appendix~\ref{app:proof_online_regret}.
\end{proposition}
We then need a consequence of the tracking and the forced exploration, proved in Appendix~\ref{app:proof_tracking}, to relate $F\big(\tw(T),\hmu(t)\big)$ to $F\big(w(T),\hmu(t)\big)$.
\begin{proposition}[Tracking]
\label{prop:tracking_tw}
Thanks to the sampling rule, precisely \eqref{eq:sampling_rule_forced_exploration} and \eqref{eq:sampling_rule_tracking}, for the choice $\gamma_t=1/(4\sqrt{t})$ it holds for all $t\geq 1$
\begin{equation}
    \left|\sum_{s=1}^t \tw(s) -N(t)\right|_{\infty}\leq 2K\sqrt{t}\,,\qquad N_a(t)\geq \frac{\sqrt{t}}{4K}-2K\ \forall a\in[1,K]\,.
    \label{eq:tracking_tw}
\end{equation}
\end{proposition}
Using Proposition~\ref{prop:regularity},\,~\ref{prop:regret_bound} and~\ref{prop:tracking_tw} one can proves that for $T\gtrsim 1/\epsilon^2$, on the event $\cE_\epsilon(T)$
\begin{equation*}
   F\big(w(T),\hmu(T)\big)\gtrsim F\big(\wstar(\mu),\mu\big)-\epsilon=T^\star(\mu)^{-1}-\epsilon\,.
\end{equation*}
Hence if we rewrite the stopping rule~\eqref{eq:stopping_rule}
\[
\frac{\beta\big(N(t),\delta\big)}{T} \leq F\big(w(T),\hmu(T)\big) \,,
\]
since $\beta\big(N(T),\delta\big)\sim \log(1/\delta)$ the algorithm will stop as soon as $T\gtrsim \log(1/\delta)/(T^\star(\mu)-\epsilon)$. Thus for such $T$ we have the inclusion $\cE_\epsilon(T) \subset \{\tau_\delta \leq T\}$. But thanks to the forced exploration, see Lemma~\ref{lem:deviation_E_epsilon}, we know that $\PP_\mu\!\big(\cE_\epsilon(T)\big) \lesssim e^{-C_{\epsilon} T^{1/16}}$. Therefore we obtain
\begin{align*}
    \EE_\mu[\tau_\delta]= \sum_{T=0}^{+\infty}\PP_\mu(\tau_\delta>T) &\lesssim  \frac{\log(1/\delta)}{T^\star(\mu)^{-1}-\epsilon}+1/\epsilon^2+\sum_{T=1}^{\infty} e^{-C_\epsilon T^{1/16}}\,.
\end{align*}
Thus dividing the above inequality by $\log(1/\delta)$  and  letting $\delta$ go to zero then $\epsilon$ go to zero allows us to conclude.

\section{Numerical Experiments}
\label{sec:experiments}
For the experiments we consider the Best Arm Identification problem described in Section~\ref{sec:BAI}. Precisely we restrict our attention to the simple, arbitrary, 4-armed bandit problem $\mu=[1,\, 0.85,\, 0.8,\, 0.75]$. The optimal proportion of draws is $\wstar(\mu)=[0.403,\,0.366,\,0.147,\, 0.083]$. The experiments compare several algorithms: the Lazy Mirror Ascent (LMA) described in Algorithm~\ref{alg:gradient_ascent}, the same algorithm but with a constant learning rate (LMAc), the Best Challenger (BC) algorithm given in Section~\ref{sec:BAI}, the Direct Tracking (DT) algorithm by \citet{garivier2016optimal}, Top Two Thompson Sampling (TTTS) by \citet{russo2016simple} and finally the uniform Sampling (Unif) as baseline. See Appendix~\ref{app:details_nume_exp} for details. Note in particular that all of them use the same Chernoff Stopping rule~\eqref{eq:stopping_rule} with the same threshold $\beta(t,\delta) = \log((\log(t)+1)/\delta)$ and the same decision rule \eqref{eq:decision_rule}. This allows a fair comparison between the sampling rules. Indeed it is known (see \citet{garivier2017thresholding}) that the choice of the stopping rule is decisive to minimize the expected number of sample. We only investigate here the effects of the sampling rule here because it is where the trade-off between uniform exploration and selective exploration takes place.
\begin{figure}[!ht]
\centering
\includegraphics[width=0.92\columnwidth]{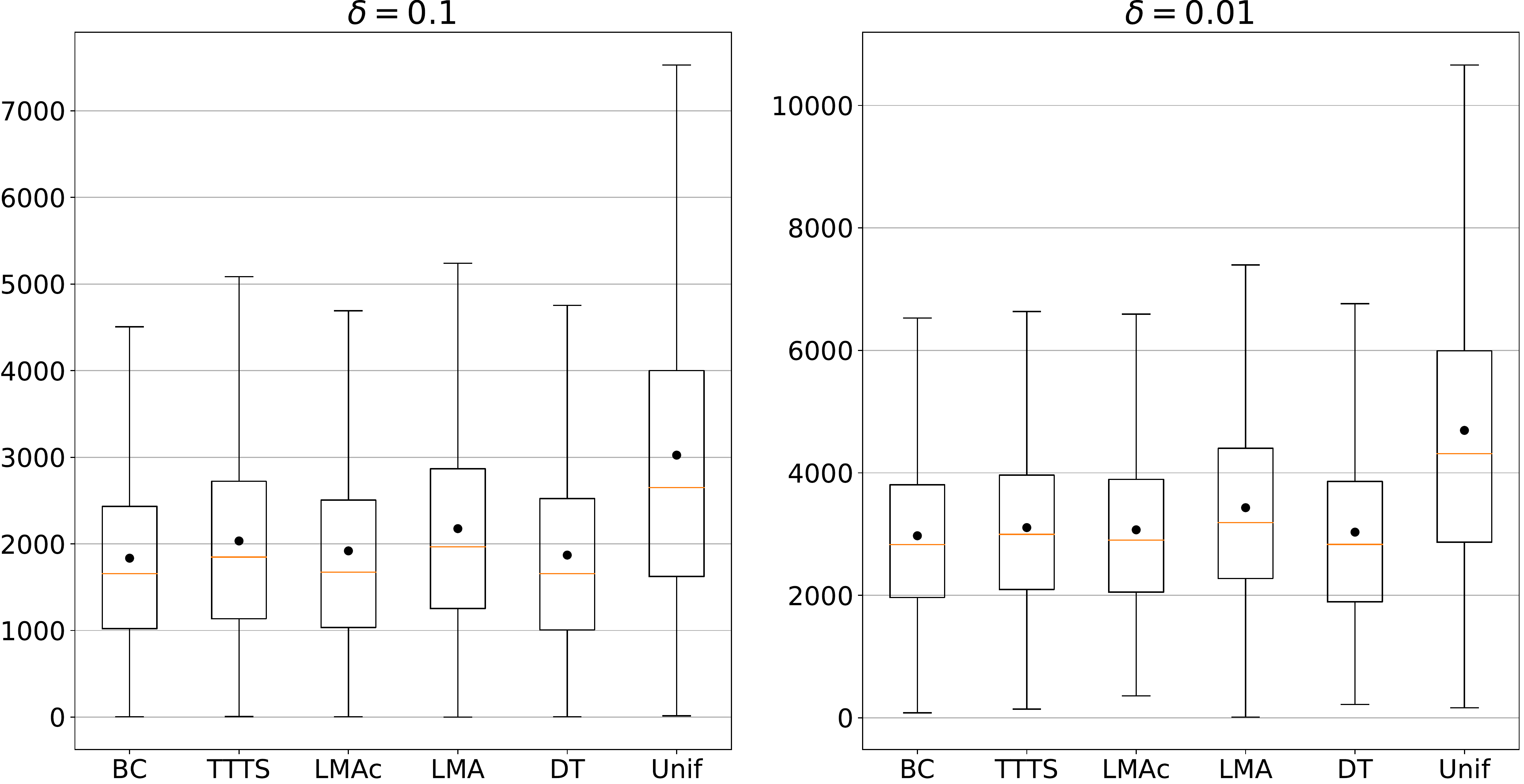}
\caption{Expected number of draws $\EE_\mu[\tau_\delta]$ (expectations are approximated over $1000$ runs) of various algorithms for the bandit problem $\mu=[1,\, 0.85,\, 0.8,\, 0.7]$. The black dots are the expected number of draws, the orange solid lines the medians.}
\label{fig:comp}
\vskip -0.2in
\end{figure}

\begin{table}[!ht]
    \centering
    \begin{tabular}{|c|c|c|c|c|c|c|}
        \hline
        Algorithm & BC & TTTS & LMAc & LMA & DT & Unif \\\hline
        Time (in second) & \num{5e-6} & \num{1e-4} &  \num{8e-6} & \num{8e-6} & \num{2e-3} & \num{4e-6} \\
        \hline
    \end{tabular}
    \caption{Average time (over 100 runs) of one step of various algorithms for the bandit problems $\mu=[1,\, 0.85,\, 0.8,\, 0.7]$.}
    \label{tab:execution_time}
\end{table}
Figure~\ref{fig:comp} displays the average number of draws of each aforementioned algorithms for two different confidence levels $\delta=0.1$ and $\delta=0.01$. The associated theoretical expected number of draws is respectively $T^\star(\mu)\log(1/\delta) \approx 1066$ for $\delta=0.1$ and $T^\star(\mu)\log(1/\delta) \approx 2133$ for $\delta= 0.01$.  Table~\ref{tab:execution_time} displays the average execution time of one step of these algorithms. Unsurprisingly all the algorithms perform better than the uniform sampling. LMA compares to the other algorithms but with slightly worse results. This may due to the fact that lazy mirror ascent (with a learning rate of order $1/\sqrt{t}$) is less aggressive than Frank Wolfe algorithm for example. Indeed using a constant learning rate (LMAc) we recover the same results as BC. But doing so we loose the guaranty of asymptotic optimality. The four mentioned algorithms share roughly the same (one step) execution time which is normal since they have the same complexity, see Appendix~\ref{app:details_nume_exp}. The Direct Tracking of the optimal proportion of draws performs slightly better than the other algorithms but the execution time is much longer (approximately 100 times longer) due to the extra cost of computing the optimal weights. Note that TTTS also tends to be slow when the posteriors are well concentrated, since it is then hard to sample the challenger. But it is the only algorithm that does not explicitly force the exploration.

\section{Conclusion}

In this paper we developed an unified approach to Bandit Active Exploration problems. In particular we provided a general, computationally efficient, asymptotically optimal algorithm. To avoid obfuscating technicalities, we treated only the case of Gaussian arms with known variance and unknown mean, but the results can easily be extended to other one-parameter exponential families. For this, we just need to replace the maximal inequality of Proposition~\ref{prop:max_ineq} by the one of Theorem 14 by \citet{kaufmann2018mixture} and to adapt the threshold accordingly.

Several questions remain open. It would be interesting to provide an analysis for the moderate-confidence regime as argued by \citet{simchowitz2017simulator}. An other way of improvement could be to explore further the connection with the Frank-Wolfe algorithm. 
Nevertheless the main open question, from the author point of view, is to find a natural way to explore instead of forcing the exploration. One possibility could be to use in this setting the principle of optimism. Because even for the Active Exploration problems there is trade-off between uniformly explore the distributions of the arms and selectively explore the distribution of specific arms to find in which set the bandit problem lies.

\bibliography{biblio-BLB}

\begin{thebibliography}{28}
\providecommand{\natexlab}[1]{#1}
\providecommand{\url}[1]{\texttt{#1}}
\expandafter\ifx\csname urlstyle\endcsname\relax
  \providecommand{\doi}[1]{doi: #1}\else
  \providecommand{\doi}{doi: \begingroup \urlstyle{rm}\Url}\fi

\bibitem[Abbasi-Yadkori et~al.(2011)Abbasi-Yadkori, P{\'a}l, and
  Szepesv{\'a}ri]{abbasi2011improved}
Yasin Abbasi-Yadkori, D{\'a}vid P{\'a}l, and Csaba Szepesv{\'a}ri.
\newblock Improved algorithms for linear stochastic bandits.
\newblock In \emph{Advances in Neural Information Processing Systems}, pages
  2312--2320, 2011.

\bibitem[Antos et~al.(2008)Antos, Grover, and Szepesv{\'a}ri]{antos2008active}
Andr{\'a}s Antos, Varun Grover, and Csaba Szepesv{\'a}ri.
\newblock Active learning in multi-armed bandits.
\newblock In \emph{International Conference on Algorithmic Learning Theory},
  pages 287--302. Springer, 2008.

\bibitem[Audibert and Bubeck(2010)]{audibert2010best}
Jean-Yves Audibert and S{\'e}bastien Bubeck.
\newblock Best arm identification in multi-armed bandits.
\newblock In \emph{COLT-23th Conference on Learning Theory-2010}, pages 13--p,
  2010.

\bibitem[Balsubramani(2014)]{balsubramani2014sharp}
Akshay Balsubramani.
\newblock Sharp finite-time iterated-logarithm martingale concentration.
\newblock \emph{arXiv preprint arXiv:1405.2639}, 2014.

\bibitem[Berthet and Perchet(2017)]{berthet2017fast}
Quentin Berthet and Vianney Perchet.
\newblock Fast rates for bandit optimization with upper-confidence frank-wolfe.
\newblock In \emph{Advances in Neural Information Processing Systems}, pages
  2225--2234, 2017.

\bibitem[Bubeck(2011)]{bubeck2011introduction}
S{\'e}bastien Bubeck.
\newblock Introduction to online optimization.
\newblock \emph{Lecture Notes}, 2011.

\bibitem[Bubeck et~al.(2012)Bubeck, Cesa-Bianchi, et~al.]{bubeck2012regret}
S{\'e}bastien Bubeck, Nicolo Cesa-Bianchi, et~al.
\newblock Regret analysis of stochastic and nonstochastic multi-armed bandit
  problems.
\newblock \emph{Foundations and Trends{\textregistered} in Machine Learning},
  5\penalty0 (1):\penalty0 1--122, 2012.

\bibitem[Bubeck et~al.(2015)]{bubeck2015convex}
S{\'e}bastien Bubeck et~al.
\newblock Convex optimization: Algorithms and complexity.
\newblock \emph{Foundations and Trends{\textregistered} in Machine Learning},
  8\penalty0 (3-4):\penalty0 231--357, 2015.

\bibitem[Chernoff(1959)]{chernoff1959sequential}
Herman Chernoff.
\newblock Sequential design of experiments.
\newblock \emph{The Annals of Mathematical Statistics}, 30\penalty0
  (3):\penalty0 755--770, 1959.

\bibitem[Combes et~al.(2017)Combes, Magureanu, and
  Proutiere]{combes2017minimal}
Richard Combes, Stefan Magureanu, and Alexandre Proutiere.
\newblock Minimal exploration in structured stochastic bandits.
\newblock In \emph{Advances in Neural Information Processing Systems}, pages
  1763--1771, 2017.

\bibitem[Degenne and Koolen(2019)]{degenne2019pure}
R{\'e}my Degenne and Wouter~M Koolen.
\newblock Pure exploration with multiple correct answers.
\newblock \emph{arXiv preprint arXiv:1902.03475}, 2019.

\bibitem[Even-Dar et~al.(2002)Even-Dar, Mannor, and Mansour]{even2002pac}
Eyal Even-Dar, Shie Mannor, and Yishay Mansour.
\newblock Pac bounds for multi-armed bandit and markov decision processes.
\newblock In \emph{International Conference on Computational Learning Theory},
  pages 255--270. Springer, 2002.

\bibitem[Finkelstein et~al.(1971)]{finkelstein1971law}
Helen Finkelstein et~al.
\newblock The law of the iterated logarithm for empirical distribution.
\newblock \emph{The Annals of Mathematical Statistics}, 42\penalty0
  (2):\penalty0 607--615, 1971.

\bibitem[Frank and Wolfe(1956)]{frank1956algorithm}
Marguerite Frank and Philip Wolfe.
\newblock An algorithm for quadratic programming.
\newblock \emph{Naval research logistics quarterly}, 3\penalty0 (1-2):\penalty0
  95--110, 1956.

\bibitem[Garivier and Kaufmann(2016)]{garivier2016optimal}
Aur{\'e}lien Garivier and Emilie Kaufmann.
\newblock Optimal best arm identification with fixed confidence.
\newblock In \emph{Conference on Learning Theory}, pages 998--1027, 2016.

\bibitem[Garivier et~al.(2017)Garivier, M{\'e}nard, and
  Rossi]{garivier2017thresholding}
Aur{\'e}lien Garivier, Pierre M{\'e}nard, and Laurent Rossi.
\newblock Thresholding bandit for dose-ranging: The impact of monotonicity.
\newblock \emph{arXiv preprint arXiv:1711.04454}, 2017.

\bibitem[Garivier et~al.(2018)Garivier, M{\'e}nard, and
  Stoltz]{garivier2018explore}
Aur{\'e}lien Garivier, Pierre M{\'e}nard, and Gilles Stoltz.
\newblock Explore first, exploit next: The true shape of regret in bandit
  problems.
\newblock \emph{Mathematics of Operations Research}, 2018.

\bibitem[Kaufmann and Koolen(2018)]{kaufmann2018mixture}
Emilie Kaufmann and Wouter Koolen.
\newblock Mixture martingales revisited with applications to sequential tests
  and confidence intervals.
\newblock \emph{arXiv preprint arXiv:1811.11419}, 2018.

\bibitem[Kaufmann et~al.(2016)Kaufmann, Capp{\'e}, and Garivier]{KaCaGa16}
Emilie Kaufmann, Olivier Capp{\'e}, and Aur{\'e}lien Garivier.
\newblock On the complexity of best-arm identification in multi-armed bandit
  models.
\newblock \emph{The Journal of Machine Learning Research}, 17\penalty0
  (1):\penalty0 1--42, 2016.

\bibitem[Lattimore and Szepesvari(2017)]{lattimore2017end}
Tor Lattimore and Csaba Szepesvari.
\newblock The end of optimism? an asymptotic analysis of finite-armed linear
  bandits.
\newblock In \emph{Artificial Intelligence and Statistics}, pages 728--737,
  2017.

\bibitem[Lattimore and Szepesv{\'a}ri(2019)]{lattimore2018bandit}
Tor Lattimore and Csaba Szepesv{\'a}ri.
\newblock \emph{Bandit Algorithms}.
\newblock Preprint, 2019.

\bibitem[Locatelli et~al.(2016)Locatelli, Gutzeit, and
  Carpentier]{LocatelliGC16}
Andrea Locatelli, Maurilio Gutzeit, and Alexandra Carpentier.
\newblock An optimal algorithm for the thresholding bandit problem.
\newblock In \emph{Proceedings of the 33nd International Conference on Machine
  Learning, {ICML} 2016, New York City, NY, USA, June 19-24, 2016}, pages
  1690--1698, 2016.

\bibitem[Mannor and Tsitsiklis(2004)]{mannor2004sample}
Shie Mannor and John~N Tsitsiklis.
\newblock The sample complexity of exploration in the multi-armed bandit
  problem.
\newblock \emph{Journal of Machine Learning Research}, 5\penalty0
  (Jun):\penalty0 623--648, 2004.

\bibitem[Pe{\~n}a et~al.(2008)Pe{\~n}a, Lai, and Shao]{pena2008self}
Victor~H Pe{\~n}a, Tze~Leung Lai, and Qi-Man Shao.
\newblock \emph{Self-Normalized Processes}.
\newblock Springer Science \& Business Media, 2008.

\bibitem[Russo(2016)]{russo2016simple}
Daniel Russo.
\newblock Simple bayesian algorithms for best arm identification.
\newblock In \emph{Conference on Learning Theory}, pages 1417--1418, 2016.

\bibitem[Shalev-Shwartz et~al.(2012)]{shalev2012online}
Shai Shalev-Shwartz et~al.
\newblock Online learning and online convex optimization.
\newblock \emph{Foundations and Trends{\textregistered} in Machine Learning},
  4\penalty0 (2):\penalty0 107--194, 2012.

\bibitem[Simchowitz et~al.(2017)Simchowitz, Jamieson, and
  Recht]{simchowitz2017simulator}
Max Simchowitz, Kevin Jamieson, and Benjamin Recht.
\newblock The simulator: Understanding adaptive sampling in the
  moderate-confidence regime.
\newblock In \emph{Conference on Learning Theory}, pages 1794--1834, 2017.

\bibitem[Soare et~al.(2014)Soare, Lazaric, and Munos]{soare2014best}
Marta Soare, Alessandro Lazaric, and R{\'e}mi Munos.
\newblock Best-arm identification in linear bandits.
\newblock In \emph{Advances in Neural Information Processing Systems}, pages
  828--836, 2014.

\end{thebibliography}
\bibliographystyle{plainnat}

\appendix

\section{Examples}
\label{app:examples}
In this appendix we present some classical and less classical active exploration bandit problems that can be described by the general framework presented in Section~\ref{sec:problem_description}. Note that for all examples presented below Assumption~\ref{assp:bounded_gradient} holds. For the three first examples it is a direct consequence of the expression of the sub-gradient. For the last one just needs to remark that the projection $\lambda$ of a certain $\mu$ on an alternative set $\cS_i$ (for $i\neq i(\mu)$) is such that $\lambda_i$ belongs to the interval $[\min_{x\in\{\mu_1,\ldots,\mu_K,S} x,\, \max_{y\in\{\mu_1,\ldots,\mu_K,S\}} y]$ for all $i\in[K]$.

\subsection{Thresholding Bandits}
\label{sec:thresholding_bandit}
We fix a threshold $\thr \in \R$.
The objective here is to identify the set of arms $a$ above this threshold, $\{ a :\, \mu_a > \thr \}$. Therefore, to see this problem as a particular case of the one presented in Section~\ref{sec:problem_description} we choose $\cI=\cP\big([1,K]\big)$ the power set of $[1,K]$ and
\[
\cS_i=\big\{\mu'\in\cM:\ \{ a :\, \mu'_a > \thr \}=i\big\}\,.
\]
For $\mu\in\cS$, it turns out that there is an explicit expression for $F$ and the characteristic time in this particular case,
\begin{equation}
    \label{eq:F_threshold}
F(w,\mu) = \min_{a\in[1,K]} w_a \d(\mu_a, \thr) \quad T^\star(\mu)=\sum_{a=1}^K\frac{1}{\d(\mu_a,\thr)}\,.
\end{equation}
In the function $F$ we recognize the minimum of the costs (with respect to the weights $w$) for moving the mean of one arm to the threshold. Thanks to this rewriting the computation of the sub-gradient is direct
\[
\nabla F(w,\mu)=
\begin{bmatrix}
(0)\\
\d(\mu_a, \thr)\\
(0)
\end{bmatrix}\!,
\]
for $a$ that realize the minimum in~\eqref{eq:F_threshold} (the non-zero coordinate is at position $a$).
\subsection{Best Arm Identification}
\label{sec:BAI}
Here the objective is to identify the arm with the greatest mean. We set $\cI=[1,K]$ and
\[
\cS_i=\big\{\mu'\in\cM:\ \mu'_i>\mu'_a,\, \forall a\neq i\big\}\,.
\]
For $\mu\in \cS_i$, we can simplify a bit the expression of the characteristic time.  Indeed, using well chosen alternatives, see \citet{garivier2016optimal}, we have
\begin{equation}
    \label{eq:F_best_arm}
    F(w,\mu)=\min_{a\neq i}\, w_i \d(\mu_i,\bmu_{i,a}^w)+w_a \d(\mu_a,\bmu_{i,a}^w)\,,
\end{equation}
where $\bmu_{i,a}^w$ is the mean between the optimal mean $\mu_i$ and the mean $\mu_a$ with respect to the weights $w$:
\[
\bmu_{i,a}^w=\frac{w_i}{w_i+w_a}\mu_i+\frac{w_a}{w_i+w_a}\mu_a\,.
\]
We can see the weighted divergence that appears in \eqref{eq:def_F} as the cost for moving the mean of arm $a$ above the optimal one $\mu_i$ and thus make the arm $a$ optimal. Precisely we move at the same time $\mu_i$ and $\mu_a$ to the weighted mean $\bmu_{i,a}^w$. The computation of the sub-gradient is also straightforward in this case
\[
\nabla F(w,\mu)=
\begin{bmatrix}
(0)\\
\d(\mu_i,\bmu_{i,a}^w)\\
(0)\\
\d(\mu_a,\bmu_{i,a}^w)\\
(0)
\end{bmatrix}\!,
\]
for active coordinates $a\neq i$ that realize the minimum in~\eqref{eq:F_best_arm} (the non-zero coordinates are at positions $i$ and $a$).
A variant of the Best Challenger sampling rule introduced by \citet{garivier2016optimal}, see also \citet{russo2016simple}, is given by 
    \begin{align}
    & C_t \in \argmin_{a\in[1,K]/i_t} w_{i_t}(t) \d\big(\hmu_{i_t}(t),\bmu_{i_t,a}^{w(t)}(t)\big)+w_a(t) \d\big(\hmu_a(t),\bmu_{i_t,a}^{w(t)}(t)\big)\nonumber\\
    &A_{t+1} =\begin{cases}
    i_t &\text{ if } \d\big(\hmu_{i_t}(t),\bmu_{i_t,C_t}^{w(t)}(t)\big)>\d\big(\hmu_{C_t}(t),\bmu_{i_t,C_t}^{w(t)}(t)\!\big)\\
    C_t &\text{ else},
    \end{cases}\label{eq:bai_best_challenger}
\end{align}
where we denote by $i_t$ the current optimal arm (the one with the greatest mean) at time $t$. At a high level, we select the best challenger $C_t$ of the current best arm $i_t$ with respect to the cost that appear in \eqref{eq:F_best_arm}. Then we greedily choose between $C_t$ and $i_t$ the one that increases the most this cost. Again, as in the previous example, this sampling rule rewrites as one step of the Frank-Wolfe algorithm for the function $F\big(\cdot,\hmu(t)\big)$
\begin{align}
    e_{A_{t+1}}&\in\argmax_{w\in\Sigma_K} w\cdot \nabla F\big(w(t),\hmu(t)\big)\nonumber\\
    w(t+1) &= \frac{t}{t+1} w(t) + \frac{1}{t+1} e_{A_{t+1}}\label{eq:def_frank_wolfe_based}\,.
\end{align}

\subsection{Signed Bandits}
\label{sec:signed_bandits}
This is a variant of the Thresholding Bandits problem where we add the assumption that all the means lie above or under a certain threshold $\thr$. Thus we choose $\cI=\{+,\,-\}$ and 
\[
\cS_+=\{\mu' \in \cM:\ \mu'_a>\thr\}\quad \cS_-=\{\mu' \in \cM:\ \mu'_a<\thr\}\,.
\]
It is easy to see, for $\mu\in\cS$, that the function $F$ and the characteristic time reduce to
\begin{equation}
    \label{eq:signed_bandit} F(w,\mu)= \sum_{a=1}^K w_a\d(\mu_a,\thr) \quad T^\star(\mu)=\frac{1}{\max_{a\in[1,K]}\d(\mu_a,\thr)}\,.
\end{equation}
In the function $F$ we recognize the cost (with respect to the weights $w$) for moving all the means to the threshold $\thr$.  The sub-gradient of $F(\cdot,\mu)$ at $w$ is 
\[
\nabla F(w,\mu)=
\begin{bmatrix}
\d(\mu_1,\thr)\\
\vdots\\
\d(\mu_a, \thr)\\
\vdots\\
\d(\mu_K, \thr)
\end{bmatrix}\!.
\]
This example is interesting because if we follow a sampling rule based on the Frank-Wolfe algorithm, see \eqref{eq:def_frank_wolfe_based} (which is equivalent to track the optimal proportion of draws in this case), it would boil down to a kind of Follow the Leader sampling rule. And it is well known that it can fail to sample asymptotically according to the optimal proportion of draws which is in this case: 
\[
\wstar_a=\begin{cases} 1/L\text{ if }a \in \argmax_b\d(\mu_b,\thr)\\
0\text{ else}
\end{cases}\,,
\]
where $L$ is the number of arms that attain the maximum that appears in the definition of the characteristic time, see~\eqref{eq:signed_bandit}. This highlights the necessity to force in some way the exploration.

\subsection{Monotonous thresholding bandit}
\label{sec:monotonous_bandit}
 It is again a variant of the Thresholding Bandit problem with some additional structure. We fix a threshold $\thr$ and assume that sequence of means is increasing. The objective is to identify the arm with the closest mean to the threshold. Hence, we choose $\cI =[1,K]$ and 
 \[\cS_i=\{\mu'\in \cM:\ \mu_1<\ldots<\mu_K,\,|\mu_i-\thr|< |\mu_a-\thr|\ \forall a\neq i \}.
 \]
Unfortunately there is no explicit expressions for $F$ neither for the characteristic time in this problem. But it is possible to compute efficiently an element of the sub-gradient of $F$ using isotonic regressions, see~\citet{garivier2017thresholding}.

\section{Details on Numerical Experiments}
\label{app:details_nume_exp}

As stated in the Section~\ref{sec:experiments} we consider the Best Arm Identification problem (see Appendix~\ref{app:examples}) for $\mu = [1, 0.85, 0.8, 0.75
]$. For all the algorithms we used the same stopping rule~\eqref{eq:stopping_rule} with the threshold $\beta(t,\delta) = \log((\log(t)+1)/\delta)$ and decision rule~\eqref{eq:decision_rule}. We consider the following sampling rules:
 \begin{itemize}[label={-}]
     \item \textit{BC}: it is the sampling rule given by~\eqref{eq:bai_best_challenger} plus forced exploration as proposed by \citet{garivier2016optimal} (if the number of pulls of one arm is less than $\sim\sqrt{t}$ then this arm is automatically sampled). The complexity of one step is of order $O(K)$, see~\eqref{eq:bai_best_challenger}.
     \item \textit{TTTS}: it is basically the sampling rule of Top Tow Thompson Sampling by \citet{russo2016simple}. We use a Gaussian prior $\Ng(0,1)$ for each arms and we slightly alter the rule to choose between the best sampled arm $I$ and its re-sampled challenger $J$. Inspired by~\eqref{eq:bai_best_challenger}, if we denote by $\mu'$ the sample from the posterior where $I$ is optimal and by $\mu''$ the re-sample where $J$ is optimal, we choose arm $I$ if $\d(\mu'_I,\mu''_I)>\d(\mu'_J,\mu''_J)$, $J$ else. Here the complexity of one step is dominated by the sampling phase, in particular the sampling of the challenger, which can be costly if the posterior are concentrated. 
     \item \textit{LMA}: this is Algorithm~\ref{alg:gradient_ascent}. We do not try to optimize the parameters. We choose a learning rate of the form $\eta_t=1/(L\sqrt{t})$ where $L$ is of order the norm of the sub-gradients and the same exploration rate $\gamma_t$ as Theorem~\ref{th:asymptotic_optimality}. The complexity of one step is of order $O(K)$ (for computing the sub-gradient).
     \item \textit{LMAc}: Exactly the same as above but with a constant learning rate.
     \item \textit{DT}: this is the Direct Tracking (DT) algorithm by \citet{garivier2016optimal}, it basically tracks the optimal weights associated to the vector of empirical means plus some forced exploration (same as BC). For the Best Arm Identification problem, to compute the optimal weights, one needs to find the root of an increasing function, e.g. by the bisection method, whose evaluations requires the resolution of K scalar equations.  
     \item \textit{Unif}: the arm is selected at random.
 \end{itemize}

\section{Proof of Theorem~\ref{th:asymptotic_optimality}}
\label{app:proof_main_result}



Fix $\epsilon>0$ some real number and consider the typical event
\begin{equation}
    \cE_\epsilon(T)=\bigcap_{t\geq g(T)}^T\big\{\hmu(t)\in\cB_{\infty}(\mu,\kappa_\epsilon)\big\}\,.
\end{equation}
where $g(T):=\floor{T^{1/4}}
$, for some horizon $T$ such that $T\geq K$ and $2 g(T)\leq T$ ($T\geq 3$ is sufficient). We also impose $T$ to be greater than the smallest integer $T_M$ such that $M\sqrt{g(T_M)}\geq L$. This condition allows to get rid of the effects of clipping the gradient on $\cE_\epsilon(T)$.

Using Proposition~\ref{prop:regularity} we can replace the vector of empirical means $\hmu(t)$ by the true vector of means $\mu$ in the first sum of~\eqref{eq:regret_bound} at cost $\epsilon T$
\begin{align*}
    \sum_{t=g(T)}^T \Big|F\big(\wstar(\mu),\hmu(t)\big)-F\big(\wstar(\mu),\mu\big) \Big|&\leq \epsilon T \,,
\end{align*}
similarly, we can replace $\hmu(t)$ by $\hmu(T)$ in the second sum   
\begin{align*}
    \sum_{t=g(T)}^T \Big|F\big(\tw(t),\hmu(t)\big)-F\big(\tw(t),\hmu(T)\big) \Big|&\leq\epsilon T\,.
\end{align*}
Hence, we deduce from~\eqref{eq:regret_bound}, with $\tT=(T-g(T)+1)$, on the event $\cE_\epsilon(T)$
\begin{equation}
    \label{eq:regret_bound_2}
    \tT F\big(\wstar(\mu),\mu\big)-\sum_{t=g(T)}^T F\big(\tw(t),\hmu(T)\big)\leq C_0\sqrt{T}+2\epsilon T\,.
\end{equation}
Now we need to compare the sum in~\eqref{eq:regret_bound_2} with the quantity $\tT F\big(w(T),\hmu(T)\big)$. To this end we will use Proposition~\ref{prop:regret_bound}, which is a consequence of the tracking and the forced exploration, see~\eqref{eq:sampling_rule_tracking} and \eqref{eq:sampling_rule_forced_exploration}.
Thus, using the concavity of $F\big(\cdot,\hmu(t)\big)$ then Proposition~\ref{prop:regularity} we have 
\begin{align*}
\sum_{t=g(T)}^T F\big(\tw(t),\hmu(T)\big) &\leq \tT F\!\!\left(\frac{1}{\tT}\sum_{t=g(T)
}^T\tw(t),\hmu(T)\right)\\
&\leq \tT F\big(w(T),\hmu(T)\big)+\tT L K  \left|w(T) -\frac{1}{\tT} \sum_{t=g(T)
}^T\tw(t) \right|_{\infty}\\
\end{align*}
Before applying Proposition~\ref{prop:tracking_tw} we need to handle the fact that the sum in the last inequality above begins at $g(T)$. But it is not harmful because $g(T)$ is small enough, one can proves:
\begin{equation}
\label{eq:get_ride_of_gT}
\left|w(T) -\frac{1}{\tT} \sum_{t=g(T)
}^T\tw(t) \right|_{\infty} \!\!\!\leq \left|w(T) -\frac{1}{T} \sum_{t=1
}^T\tw(t) \right|_{\infty}\!\!\!\!+ \frac{2}{\sqrt{T}}.
\end{equation}
Indeed, using the triangular inequality we have 
\begin{align*}
    \left|w(T)- \frac{1}{\tT} \sum_{t=g(T)}^T\tw(t)\right|_{\infty}\leq  \left|w(T)- \frac{1}{T} \sum_{t=1}^T\tw(t)\right|_{\infty}+  \left|\frac{1}{T} \sum_{t=1}^T\tw(t)- \frac{1}{\tT} \sum_{t=g(T)}^T\tw(t)\right|_{\infty}\,.
\end{align*}
 It remains to notice that
  \begin{align*}
\left|\frac{1}{T} \sum_{t=1}^T\tw(t) -\frac{1}{\tT} \sum_{t=g(T)
}^T\tw(t) \right|_{\infty}&\leq \left|\frac{1}{T} \sum_{t=1}^T\tw(t) -\frac{1}{T} \sum_{t=g(T)
}^T\tw(t) \right|_{\infty}+ \left|\frac{1}{T} \sum_{t=g(T)}^T\tw(t) -\frac{1}{\tT} \sum_{t=g(T)
}^T\tw(t) \right|_{\infty}\\
&\leq \frac{g(T)}{T}+ \left(\frac{1}{\tT}-\frac{1}{T}\right)\tT\leq 2\frac{g(T)}{T}\\
&\leq \frac{2}{\sqrt{T}}\,,
\end{align*}
where in the last line we used $g(T)\leq \sqrt{T}$, by definition.
Now, using~\eqref{eq:get_ride_of_gT} then~\eqref{eq:tracking_tw} we obtain 
\[
\sum_{t=K}^T F\big(\tw(t),\hmu(T)\big) \leq \tT F\big(w(T),\hmu(T)\big)+\tT \frac{4 L K^2 }{\sqrt{T}}\,.
\]
Thus, using the above inequality in~\eqref{eq:regret_bound_2} and dividing by $\tT$ we get  
\begin{align*}
    F\big(\wstar(\mu),\mu\big)-\!F\big(w(T),\hmu(T)\big)\!&\leq \frac{C_0 \sqrt{T}}{\tT}+ \frac{2\epsilon T}{\tT}+\frac{4 L K^2}{\sqrt{T}}\\
    &\leq \underbrace{(2C_0\! +\! 4 K^2 L\big)}_{:=C_1}\frac{1}{\sqrt{T}}+4\epsilon,
\end{align*}
where in the last line we used $T/\tT\leq 2$, thanks to the choice of $T$. For $T\geq (C_1/\epsilon)^2$, we finally obtain the bound announced in Section~\ref{sec:gradient_ascent_proof}
\begin{equation}
   F\big(w(T),\hmu(T)\big)\geq F\big(\wstar(\mu),\mu\big)-5\epsilon=T^\star(\mu)^{-1}-5\epsilon\,.
\end{equation}
Hence the algorithm will stop at $T$ if $\beta\big(N(T),\delta\big)/T\leq T^\star(\mu)^{-1}-5\epsilon $. We use the following technical lemma (proved in Appendix~\ref{app:other_proofs}) to characterize such $T$.
\begin{lemma}
There exits a constant $C_3(\epsilon)$ that depends on $\epsilon$ and $K$, such that for 
\[
T\geq \frac{\log(1/\delta)+K\log\!\big(4\log(1/\delta)+1\big)}{T^\star(\mu)^{-1}-6\epsilon} +C_3(\epsilon)\,,
\]
it holds 
\[
\beta\big(N(T),\delta\big)/T\leq  T^\star(\mu)^{-1}-5\epsilon\,.
\]
\label{lem:invers_log_log}
\end{lemma}
We also need to use that $\cE_\epsilon(T)$ is a typical event. Quantitatively, using the consequence of the forced exploration~\eqref{eq:tracking_tw}, we can prove the following deviation inequality, see Appendix~\ref{app:other_proofs}.
\begin{lemma}
There exists two constants $C_4(\epsilon)$ and $C_5(\epsilon)$, that depend on $\epsilon$, $\mu$ and $K$, such that
\[
\PP_\mu\big(\cE_\epsilon(T)^{c}\big) \leq C_5(\epsilon) T e^{-C_4(\epsilon) T^{1/8}}\,.
\]
\label{lem:deviation_E_epsilon}
\end{lemma}

Putting all together, for $T$ large enough, for example:
\begin{align*}
T\geq & \frac{\log(1/\delta)+K\log\!\big(4\log(1/\delta)+1\big)}{T^\star(\mu)^{-1}-6\epsilon} +C_3(\epsilon)+ (C_1/\epsilon)^2 +K+ T_M +3\,,
\end{align*}
we have the inclusion $\cE_\epsilon(T) \subset \{\tau_\delta \leq T\}$, hence using Lemma~\ref{lem:deviation_E_epsilon}
\[
\PP_\mu(\tau_\delta > T)\leq \PP\big( \cE_{\epsilon}(T)^c \big)\leq C_5(\epsilon) T e^{-C_4(\epsilon) T^{1/8}}\,.
\]
It remains to remark that, using the above inequalities, 
\begin{align}
    \EE_\mu[\tau_\delta]= \sum_{T=0}^{+\infty}\PP_\mu(\tau_\delta>T) &\leq  \frac{\log(1/\delta)+K\log\!\big(4\log(1/\delta)+1\big)}{T^\star(\mu)^{-1}-6\epsilon}+C_3(\epsilon)+ (C_1/\epsilon )^2 +K\nonumber\\
    &+T_M +3+\sum_{T=1}^{\infty}C_5(\epsilon) T e^{-C_4(\epsilon) T^{1/8}} \!.\label{eq:presque}
\end{align}
Thus dividing~\eqref{eq:presque} by $\log(1/\delta)$  and  letting $\delta$ go to zero, we obtain 
\[
\limsup_{\delta \rightarrow 0}\frac{\EE_\mu[\tau_\delta]}{\log(1/\delta)}\leq \frac{1}{T^\star(\mu)^{-1}-6\epsilon}\,,
\]
letting $\epsilon$ go to zero allows us to conclude.

\section{Deviations Inequality}
\label{app:deviations}
Let $\theta$ be a certain parameter in $\R^d$. We consider the linear model 
\[
X_t= \theta\cdot A_t+\eta_t\,,
\]
where $\{\eta_t\}_{t\in\N^\star}$ are i.i.d. from a Gaussian distribution $\Ng(0,1)$ and $A_t\in\R^d$ is a random variable $\sigma(A_1, X_1,\ldots,A_{t-1}, X_{t-1})$-measurable. 
Let $V_t:=\sum_{t=1}^t A_s A_s^\top$ be the Gram matrix and $\htheta_t$ be the least square estimator of $\theta$ (defined when $V_t$ is invertible)
\[
\htheta_t=V_t^{-1}\sum_{s=1}^t A_s X_s\,.
\]
We assume that $A_s=e_s$ the $s$-nth vector of the canonical basis of $\R^d$ for $1 \leq s \leq d$, such that $V_t$ is invertible for $t\geq d$. We want to prove a maximal inequality on the self-normalized following quantity 
\begin{equation}
    \label{eq:def_S_t_V_t}
    \frac{|\htheta_t-\theta|^2_{V_t}}{2}=\frac{|S_t|^2_{V_t^{-1}}}{2}\,,
\end{equation}
where $S_t:=\sum_{s=1}^t A_s \eta_s$ and $|x|_V:= x^{\top} V x$ is the norm induced by the symmetric positive definite matrix $V$.
In addition we will assume that for all $t\geq1$ the random variable $A_t\in (e_l)_{l\in[1,d]}$ is an element of the canonical basis. Thus the Gram matrix $V_t$ is diagonal and for all $l\in [1,d]$
\[
V_{t,l,l}= N_{t,l}:=\sum_{s=1}^t \ind_{\{A_s = e_l\}}\,.
\]
\begin{proposition}
For $\delta>0$ and $1>\beta>0$,
\begin{equation}
    \PP\left(\exists n\geq t\geq d,\, |S_t|_{V_t^{-1}}^2/2 \geq \log(1/\delta)+(1+\beta)d\loglog(n)+o_{\delta,\beta}\big(\loglog(n)\big)\right)\leq \delta\,,
    \label{eq:max_ineq_diag_n}
\end{equation}
see the end of the proof for an explicit formula. And if we do not care about the constant in front of the term in $\loglog(n)$, it holds
\begin{equation}
    \PP\left(\exists t\geq d,\, |S_t|_{V_t^{-1}}^2/2 \geq \log(1/\delta)+6\sum_{l=1}^d \log\!\big(\log(N_{t,l})+3\big)+d\widetilde{C}\right)\leq \delta\,,
    \label{eq:max_ineq_diag_N}
\end{equation}
see the end of the proof for an explicit expression of the constant $\widetilde{C}$.
\label{prop:max_ineq_diag}
\end{proposition}
Proposition~\ref{prop:max_ineq} is a simple rewriting of \eqref{eq:max_ineq_diag_N} for $d=K$. Indeed, the Kullback-Leibler divergence in \eqref{eq:def_S_t_V_t} rewrites with the diagonal assumption on the Gram matrix
\begin{equation}
    \label{eq:S_t_V_t_in_chernoff}
    \frac{|S_t|^2_{V_t^{-1}}}{2}=\sum_{l=1}^d N_{t,l}\d(\htheta_{t,l},\theta_l)\,.
\end{equation} 
The constant in front of the $\loglog(n)$ in~\eqref{eq:max_ineq_diag_n} is optimal when $\beta$ goes to $0$ with respect to the Law of the Iterated Logarithm, for the particular case of uniform sampling, i.e. $A_t= t \mod K$, see Lemma~2 of \citet{finkelstein1971law}. The proof of Proposition~\ref{prop:max_ineq_diag} is a variation on the method of mixtures, see \citet{pena2008self} for an introduction to the method, \citet{lattimore2018bandit} and \citet{abbasi2011improved} for the use of this methods in the bandit setting. It turns out that the prior used is really close to the one used by \citet{balsubramani2014sharp} in their proof of Lemma~12.

\begin{proof}[Proof Proposition~\ref{prop:max_ineq_diag}]
 We will use the method of mixtures with the prior on $\R^d$
\[
\tf(\lambda)= \prod_{l=1}^d f(\lambda_l)\,,
\]
with $f$ a density on $\R$ given by 
\[
f(\lambda)=\frac{C_\beta}{|\lambda|\Big(\big|\log|\lambda|\big|+2\Big)^{1+\beta}}\,,
\]
where $C_\beta$ is the normalizing constant. Hence, we consider the martingale
\[
M_t=\int e^{\lambda\cdot S_t -|\lambda|^2_{V_t}/2} \tf(\lambda)\diff{\lambda}\,.
\]
We can rewrite this martingale to make appear the quantity of interest
\[
M_t= e^{|S_t|_{V_t^{-1}}^2/2}\prod_{l=1}^d \int e^{-(S_{t,l}/N_{t,l}-\lambda_l)^2 N_{t,l}/2}f(\lambda_l)\diff{\lambda}\,.
\]
Using that $f$ is symmetric and non-increasing on $\R^+$, we can lower bound the martingale as follows
\begin{align*}
    M_t  &\geq e^{|S_t|_{V_t^{-1}}^2/2}\prod_{l=1}^d \int_{S_{t,a}/N_{t,a}-\sqrt{2/N_{t,a}}}^{S_{t,a}/N_{t,a}+\sqrt{2/N_{t,a}}} e^{-(S_{t,l}/N_{t,l}-\lambda_l)^2 N_{t,l}/2}f(\lambda_l)\diff{\lambda}\\
    &\geq e^{|S_t|_{V_t^{-1}}^2/2}\prod_{l=1}^d \frac{2C_\beta e^{-1}}{\big(|S_{t,a}|/\sqrt{2 N_{t,a}} +1\big) \left(\Big|\log\big(|S_{t,a}|/N_{t,a})+\sqrt{2/N_{t,a}}\big)\Big|+2 \right)^{1+\beta}}\,.
\end{align*}

Thanks to the method of mixtures this lower bound leads to the following maximal inequality
\begin{align}
    \PP \Bigg( \exists t\geq d,\, &|S_t|_{V_t^{-1}}^2/2 \geq \log(1/\delta)+\sum_{l=1}^d \log\big(|S_{t,a}|/\sqrt{2 N_{t,a}} + 1\big)+\nonumber\\
    &(1+\beta)\sum_{l=1}^d \log\!\left(\Big|\log\big(|S_{t,a}|/N_{t,a}+\sqrt{2/N_{t,a}}\big)\Big|+2 \right)+ d\Big(1+\log\big(1/(2C_\beta)\big)\Big)\Bigg)\leq \delta\label{ineq_raw_diag}
\end{align}
We can simplify a bit the expression in \eqref{ineq_raw_diag} using that
\begin{align*}
    \log\!\left(\Big|\log\big(|S_{t,a}|/N_{t,a})+\sqrt{2/N_{t,a}}\big)\Big|+2 \right)&\leq \log\!\left(\big|\log(\sqrt{N_{t,a}/2})\big|+2+\log\big(|S_{t,a}|/\sqrt{2 N_{t,a}}+1\big) \right)\\
    &\leq \log\big(\log(N_{t,a})+3\big)+\frac{\log\big(|S_{t,a}|/\sqrt{2 N_{t,a}}+1\big)}{2}\,,
\end{align*}
where we used in the last line the fact that $\log(x+y)\leq \log(x)+y/x$ for $x,y >0$. Indeed, injecting this inequality in \eqref{ineq_raw_diag} we obtain 
\begin{align}
    \PP \Bigg( \exists t\geq d,\, |S_t|_{V_t^{-1}}^2/2 \geq \log(1/\delta)&+\sum_{l=1}^d 2\log\big(|S_{t,l}|/\sqrt{2 N_{t,l}} + 1\big)+\nonumber\\
    &(1+\beta)\sum_{l=1}^d \log\!\big(\log(N_{t,l})+3\big)+ d\Big(1+\log\big(1/(2C_\beta)\big)\Big)\Bigg)\leq \delta\label{ineq_refined_diag}
\end{align}
Now we will bootstrap this inequality to get rid of the $|S_{t,l}|/\sqrt{2N_{t,l}}$ inside the $\log$. Noting that by concavity of the logarithm and $\log(x+1)\leq x/2+\log(2)$ for $x>0$
\begin{align*}
    \sum_{l=1}^d 2\log\big(|S_{t,l}|/\sqrt{2 N_{t,l}} + 1\big)&\leq \sum_{l=1}^d \log\big(|S_{t,l}|^2/(2 N_{t,l}) + 1\big)+d\log(2)\\
    &\leq d \log\big(|S_t|_{V_t^{-1}}^2/(2 d) + 1\big)+d\log(2)\\
    & \leq |S_t|_{V_t^{-1}}^2/4+2d\log(2)
\end{align*}
we can degrade \eqref{ineq_refined_diag}, with the choice $\beta=0.5$, to 
\begin{align*}
    \PP \Bigg( \exists t\geq d,\, |S_t|_{V_t^{-1}}^2/4 \geq \log(1/\delta)+
   2\sum_{l=1}^d \log\!\big(\log(N_{t,l})+3\big)+ d\big(1+2\log(1/C_{1/2})\big)\Bigg)\leq \delta
\end{align*}
This last inequality implies the following one 
\begin{equation}
    \label{ineq_bootstrap_diag}
       \PP \Bigg( \exists t\geq d,\, |S_t|_{V_t^{-1}}^2/2 \geq 4 \log\!\left(\frac{\prod_{l=1}^d \big(\log(N_{t,l})+3\big) C}{\delta} \right)\Bigg)\leq \delta\,,
\end{equation}
where $C$ is a constant such that $\log(C)=1-2\log(C_{1/2})$. Let $A$ be the event that appears in~\eqref{ineq_refined_diag} with $\delta/2$ instead of $\delta$, $B$ be the event that appears in~\eqref{ineq_bootstrap_diag} with $\delta/2$ instead of $\delta$ and $D$ be such that 
\begin{align*}
    D:=\bigg\{\exists t\geq d,\ |S_t|_{V_t^{-1}}^2/2 &\geq \log(2/\delta)+\sum_{l=1}^d 2\log\big(|S_{t,l}|/\sqrt{2 N_{t,l}} + 1\big)+\nonumber\\
&d \log\left(4/d \log\!\left(2\frac{\prod_{l=1}^d \big(\log(N_{t,l})+3\big) C}{\delta} \right) + 1\right)+ d\Big(1+2\log\big(1/(C_\beta)\big)\Big)\Big\}\,.
\end{align*}
By~\eqref{ineq_refined_diag} and~\eqref{ineq_bootstrap_diag}, it holds 
\begin{align*}
    \PP(D)&\leq \PP(D\cap B^c) +\PP(B)\\
    &\leq \PP(A)+\PP(B)\leq \delta\,.
\end{align*}
We just proved that 
\begin{align}
       \PP \Bigg( \exists t\geq d,\, &|S_t|_{V_t^{-1}}^2/2 \geq \log(2/\delta)+(1+\beta)\sum_{l=1}^d \log\!\big(\log(N_{t,l})+3\big)+\nonumber\\
    & d \log\left(\frac{4}{d} \log\!\left(2\frac{\prod_{l=1}^d \big(\log(N_{t,l})+3\big) C}{\delta} \right) + 1\right)+ d\Big(1+2\log\big(1/(C_\beta)\big)\Big)\Bigg)\leq \delta\label{ineq_final_diag} \,.
\end{align}
To conclude we will specify \eqref{ineq_final_diag} in two ways. First if $t\leq n$, using that in this case $N_{t,l}\leq n$, we obtain 
\begin{align*}
 \PP \Bigg( \exists d\leq t\leq n,\, &|S_t|_{V_t^{-1}}^2/2 \geq \log(2/\delta)+(1+\beta)d \log\!\big(\log(n)+3\big)+\nonumber\\
    & d \log\left(\frac{4}{d}\log\!\left(2\frac{\big(\log(n)+3\big) C}{\delta} \right) + 1\right)+ d\Big(1+2\log\big(1/(C_\beta)\big)\Big)\Bigg)\leq \delta\,.
\end{align*}
And using again $\log(x+y)\leq \log(x) +x/y$, for $\beta=1/2$ and $\widetilde{C}:= 5\log(2C)$ we get
\[
 \PP\left(\exists t\geq d,\, |S_t|_{V_t^{-1}}^2/2 \geq \log(1/\delta)+d\log\!\big(4\log(1/\delta)+1\big)+6\sum_{l=1}^d \log\!\big(\log(N_{t,l})+3\big)+d\widetilde{C}\right)\leq \delta\,.
\]
\end{proof}
\section{Tracking results}
\label{app:proof_tracking}
This section is devoted to prove Proposition~\ref{prop:tracking_tw}. We will need one tool extracted from \citet{garivier2016optimal}, namely the next tracking lemma  which corresponds to Lemma~15 of the aforementioned reference.
\begin{lemma}
For all $t\geq 1$
\begin{equation*}
    \left| \sum_{s=1}^t w'(s) -N(t) \right|_{\infty}\leq K\,.
\end{equation*}
\label{lem:tracking}
\end{lemma}
\begin{proof}[Proof of Proposition~\ref{prop:tracking_tw}]
 Thanks to Lemma~\ref{lem:tracking} and the definition of the weights in \eqref{eq:sampling_rule_forced_exploration} we have
 \begin{align*}
     \left| \sum_{s=1}^t \tw(s)- N(t) \right|_{\infty}&=\left|\sum_{s=1}^t w'(s) -\frac{\gamma_s}{1-\gamma_s}\pi+\frac{\gamma_s}{1-\gamma_s}w'(s) -N(t) \right|_{\infty}\\
     &\leq \left|\sum_{s=1}^t w'(s) -N(t) \right|_{\infty}+\sum_{s=1}^t\frac{\gamma_s}{1-\gamma_s}|\pi - w'(s)|_{\infty}\\
     &\leq K+\sum_{s=1}^t 2 \gamma_s\\
     &\leq K+\sqrt{T}\leq 2 K \sqrt{T}\,,
 \end{align*}
 where in the last lines we used that $\gamma_t=1/(4\sqrt{t})$ and a comparison series integral. This proves the first part of the proposition, for the second part we just use that $w'_a(t)\geq \gamma_t/K $ and Lemma~\ref{lem:tracking},
 \begin{align*}
     N_a(t)&\geq \sum_{s=1}^t w_a'(s) -\left| N_a(t) - \sum_{s=1}^t w'_a(s) \right|\\
     &\geq \sum_{s=1}^t\frac{\gamma_s}{K}-K\geq\frac{\sqrt{t+1}-1}{2K}-K \geq \frac{\sqrt{t}}{4 K}-2K\,.
 \end{align*}
\end{proof}
\section{Online Concave Optimization}
\label{app:proof_online_regret}
We consider the classical setting of online optimization on the simplex $\Sigma_K$. Consider a sequence of gain $f_t \in \R^K$ such that $0\leq f_{t,a}\leq C_t$, for some constant $C_t$. The objective is to minimize the regret against any constant strategy $\wstar\in\Sigma_K$,
\[
\sum_{t=1}^T f_t\cdot (\wstar-w_t)\,.
\]
To this aim we can use the Exponential Weights algorithm: let $w_1=\pi$ be the uniform distribution and define the other weights as follow
\[
 w_{t+1} = \argmax_{w\in\Sigma_K} \eta_{t+1} \sum_{s=1}^{t} w \cdot f_s-\kl(w,\pi)\,,
\]
where $\eta_t$ is the learning rate. There is a closed formula for these weights
\begin{equation}
    \label{eq:closed_formula_exp}
    w_{t+1,a}=\frac{e^{\eta_{t+1} G_{t,a}}}{\sum_{b=1}^K e^{\eta_{t+1} G_{t,b}}}\,,
\end{equation}
where $G_t=\sum_{s=1}^t f_s$ with the convention $G_0=0$. The next lemma is a simple adaptation of the Theorem~2.4 of \cite{bubeck2011introduction}. We add its proof for the sake of completeness.
\begin{lemma}
If $\eta_t$ is non-increasing, for all $\wstar\in\Sigma_K$,
\[
\sum_{t=1}^T f_t\cdot (\wstar-w_t)\leq \frac{\log(K)}{\eta_T}+\sum_{t=1}^T 2\eta_t C_{t}^2\,.
\]
\label{lem:regret_online_linear}
\end{lemma}

\begin{proof}
We decompose the following quantity in two terms 
\begin{equation}
\label{eq:decomposition_regret}
    -w_t\cdot f_t = \frac{1}{\eta_t} \log \EE_{a\sim w_t} e^{\eta_t (f_{t,a}-\EE_{b\sim w_t}f_{t,b})}-\frac{1}{\eta_t}\log\EE_{a\sim w_t} e^{\eta_t f_{t,a}}\,.
\end{equation}
To bound the first term we use the Hoeffding inequality
\begin{equation}
\label{eq:regret_first_term}
    \frac{1}{\eta_t} \log \EE_{a\sim w_t} e^{\eta_t (f_{t,a}-\EE_{b\sim w_t}f_{t,b})}\leq 2\eta_t C_t^2\,.
\end{equation}
For the second term, we consider the potential function 
\[
\Phi_t(\eta)=\frac{1}{\eta}\log\left( \frac{1}{K} \sum_{a=1}^{K} e^{\eta G_{t,a}}\right)\,,
\]
with the convention $\Phi_0(t)=0$. Thanks to~\eqref{eq:closed_formula_exp} we have 
\begin{align}
-\frac{1}{\eta_t}\log\EE_{a\sim w_t} e^{\eta_t f_{t,a}} &=-\frac{1}{\eta_t}\log\frac{\sum_{a=1}^K e^{\eta_t G_{t, a}}}{\sum_{a=1}^K e^{\eta_t G_{t-1, a}}} \nonumber\\
&= \Phi_{t-1}(\eta_t) -\Phi_t(\eta_t) \label{eq:regret_second_term}\,.
\end{align}
Putting together \eqref{eq:decomposition_regret}, \eqref{eq:regret_first_term}, \eqref{eq:regret_second_term} and summing over $t$ we obtain
\[
\sum_{t=1}^T f_t\cdot (\wstar-w_t)\leq \sum_{t=1}^T 2\eta_t C_t^2+\sum_{t=1}^T \big(\Phi_{t-1}(\eta_t)-\Phi_{t}(\eta_t)\big)+\sum_{t=1}^T f_t\cdot \wstar\,.
\]
An Abel transformation on the penultimate term of the previous inequality leads to 
\[
\sum_{t=1}^T \big(\Phi_{t-1}(\eta_t)-\Phi_{t}(\eta_t)\big) = \sum_{t=1}^{T-1}\big( \Phi_t(\eta_{t+1})-\Phi_t(\eta_t)\big)-\Phi_T(\eta_T)\,,
\]
where we used that $\Phi_0(\eta_1)=0$. Since it holds that
\begin{align*}
    -\Phi_T(\eta_T)=\frac{1}{\eta_T}\log(K)-\frac{1}{\eta_T}\log\left(\sum_{a=1}^K e^{\eta_T G_{T,a}}\right)
    \leq \frac{1}{\eta_T}\log(K)-\max_{a\in[1,K]}G_{T,a}\,,
\end{align*}
we get 
\[
\sum_{t=1}^T f_t\cdot (\wstar-w_t)\leq \frac{\log(K)}{\eta_T}+\sum_{t=1}^T 2\eta_t C_t^2+\sum_{t=1}^{T-1} \big(\Phi_{t}(\eta_{t+1})-\Phi_{t}(\eta_t)\big)\,.
\]
To conclude it remains to show that $\Phi_t(\cdot)$ is non-decreasing for all $t$ since $\eta_t$ is non-increasing. To this end we just check that $\Phi'_t(\eta)\geq 0$,
\begin{align*}
    \Phi_t'(\eta)&=\frac{-1}{\eta^2}\log\left(\frac{1}{K}\sum_{a=1}^K e^{\eta G_{a,t}}\right)+\frac{1}{\eta}\frac{\sum_{a=1}^{K}e^{\eta G_{t,a}}G_{t,a}}{\sum_{a=1}^{K}e^{\eta G_{t,a}}}\\
    &= \frac{1}{\eta^2}\kl(w_t^\eta,\pi)\geq 0\,,
\end{align*}
where $w_{t,a}^\eta= e^{\eta G_{t,a}}/(\sum_{b=1}^K e^{\eta G_{t,b}})$.
\end{proof}
We are now ready to prove Proposition~\ref{prop:regret_bound}.
\begin{proof}[Proof of Proposition~\ref{prop:regret_bound}]
We will use Lemma~\ref{lem:regret_online_linear} with the choices 
\begin{equation*}
C_t=\begin{cases}
M\sqrt{t}&\text{ if } t< g(T)\\
L &\text{ else}
\end{cases},
\qquad \eta_t=\frac{1}{\sqrt{t}}\,.
\end{equation*}
Indeed thanks to Assumption~\ref{assp:bounded_gradient} and the definition of $\cE_\epsilon(T)$ we know that $0\leq\nabla_a F\big(\tw(t),\hmu(t)\big)\leq C_t $ on this event. Therefore, using Lemma~\ref{lem:regret_online_linear} up to a translation of all the indices by $K-1$, we obtain the following regret bound 
\begin{equation}
\label{eq:regret_beg_1}
    \sum_{t=K}^T \clip_s\!\Big(\nabla F \big(\tw(t),\hmu(t)\big)\Big)\cdot\big(\wstar(\mu)-\tw(t)\big)\leq \log(K)\sqrt{T}+\sum_{t=K}^{T}\frac{2 C_{t}}{\sqrt{t}}\,.
\end{equation}
It remains to control the terms inside the sums for $t\leq g(T)$. Using that the clipped sub-gradient is bounded by $C_t$ and Holder's inequality, we have
\begin{align}
    \sum_{t=K}^{g(T)-1}\Big|\clip_s\!\Big(\nabla F \big(\tw(t),\hmu(t)\big)\Big)\cdot\big(\wstar(\mu)-\tw(t)\big)\Big|&\leq \sum_{t=K}^{g(T)-1} K M \sqrt{t}\leq K M \int_{x=0}^{T^{1/4}}  \sqrt{x}\diff{x}\nonumber\\
    &=  K M \frac{2 T^{3/8}}{3}\leq K M \sqrt{T}\,.
    \label{eq:first_sum_regret_gt}
\end{align}
Similarly, one obtains, using the definition of $C_t$
\begin{align}
    \sum_{t=1}^T \frac{2 C_t^2}{\sqrt{t}}&\leq  \sum_{t=g(T)}^T \frac{2 L^2}{\sqrt{t}}+2 M^2\sum_{t=1}^{g(T)-1}  \sqrt{t}\nonumber\\
    &\leq \int_{0}^T \frac{2L^2}{\sqrt{x}}\diff{x}+2M^2\int_{0}^{T^{1/4}}\sqrt{x}\diff{x}\nonumber\\
    &=4L^2\sqrt{T}+\frac{4 M^2}{3}T^{3/8}\leq (4 L^2+2 M^2)\sqrt{T}.
    \label{eq:second_sum_regret_gt}
\end{align}
Thus, combining \eqref{eq:regret_beg_1}, \eqref{eq:first_sum_regret_gt} and \eqref{eq:second_sum_regret_gt}, we get 
\begin{equation*}
    \sum_{t=g(T)}^T \nabla F \big(\tw(t),\hmu(t)\big)\cdot\big(\wstar(\mu)-\tw(t)\big)\leq \underbrace{(\log(K)+K M + 4 L^2+2 M^2)}_{:=C_0}\sqrt{T}\,.
\end{equation*}
Note that the clipping has no effects since $t\geq g(T)$ and $T\geq T_M $. The concavity of $F\big(\cdot,\hmu(t)\big)$ allows us to conclude
\[\sum_{t=g(T)}^T  F\big(\wstar,\hmu(t)\big)- F\big(\tw(t),\hmu(t)\big)\leq C_0\sqrt{T}\,.
\]
\end{proof}

\section{Other Proofs}
\label{app:other_proofs}
We regroup in this section proofs of auxiliary results.
\subsection{Technical lemmas}

\begin{proof}[Proof of Lemma~\ref{lem:invers_log_log}]
Let $C_3(\epsilon)>0$ a constant that depends on $\epsilon$ and $K$ be such that for all $T\geq C_3(\epsilon)$,
\begin{equation*}
    6K \log\big(\log(T)+3\big)+K\widetilde{C}\leq \epsilon T\,.
\end{equation*}
Then, using that $N_a(T)\leq T$, for all $a$, for 
\[
T\geq \max\left(C_3(\epsilon),\frac{\log(1/\delta)+K\log\!\big(4\log(1/\delta)+1\big)}{T^\star(\mu)^{-1}-6\epsilon}\right)\,,
\]
 it holds 
 \begin{align*}
     \frac{\beta\big(N(T),\delta\big)}{T}&\leq \frac{\log(1/\delta)+K\log\!\big(4\log(1/\delta)+1\big)+6 K \log\big(\log(T)+3)+K\widetilde{C}}{T}\\
     &\leq \frac{\log(1/\delta)+K\log\!\big(4\log(1/\delta)+1\big)}{T}+\epsilon\\
     &\leq T^\star(\mu)^{-1}-5\epsilon\,,
 \end{align*}
 which concludes the proof.
 \end{proof}

 \begin{proof}[Proof of Lemma~\ref{lem:deviation_E_epsilon}]
 It is an adaptation of the proof of Lemma~19 by \citet{garivier2016optimal} with the Chernoff inequality for Gaussian distributions. We have 
 \begin{align*}
     \PP\mu\big(\cE_\epsilon(T)^{c}\big)& \leq \sum_{t=g(T)}^{T}\PP_\mu\big(\hmu(t)\notin \cB_{\infty}(\mu,\kappa_\epsilon)\big)\\
     &= \sum_{t=g(T)}^T\sum_{a=1}^K \left(\PP_\mu\big(\hmu_a(t)\leq \mu_a-\kappa_\epsilon\big)+\PP_\mu\big(\hmu_a(t)\geq \mu_a+\kappa_\epsilon\big)\right)\,.
 \end{align*}
 Thanks to~\eqref{eq:tracking_tw} we know that for all $a$, $\sqrt{t}/(4K)-2K\leq N_a(t)\leq t$. Let denote by $\hmu_{a,n}$ the empirical mean of the first $n$ samples from arm $a$ (such that $\hmu_a(t)=\hmu_{a,N_a(t)}$). Using the union bound then Chernoff inequality, we get 
 \begin{align*}
     \PP_\mu\big(\hmu_a(t)\leq \mu_a-\kappa_\epsilon\big)&\leq \sum_{\sqrt{t}/(4K)-2K\leq n\leq t} \PP_\mu(\hmu_{a,n}\leq \mu_a-\kappa_\epsilon)\\
     &\leq \sum_{\sqrt{t}/(4K)-2K\leq n\leq t} e^{-n \kappa_\epsilon^2/2}\leq \frac{e^{-(\sqrt{t}/(4K)-2K) \kappa_\epsilon^2/2}}{1- e^{-\kappa_\epsilon^2/2}}\\
     &\leq \frac{2}{\kappa_\epsilon^2}e^{-(\sqrt{t}/(4K)-2K-1) \kappa_\epsilon^2/2}\,.
 \end{align*}
similarly 
\[
\PP_\mu\big(\hmu_a(t)\geq \mu_a+\kappa_\epsilon\big)\leq \frac{2}{\kappa_\epsilon^2}e^{-(\sqrt{t}/(4K)-2K-1) \kappa_\epsilon^2/2}\,.
\]
Thus for the choice of the constants $C_4(\epsilon)$ and $C_5(\epsilon)$
\[
C_4(\epsilon):= \frac{\kappa_\epsilon^2}{16K}\qquad C_5(\epsilon):= \frac{4 K}{\kappa_\epsilon^2}e^{(2K+1) \kappa_\epsilon^2/2}
\]
it holds
\[
\PP_\mu\big(\cE_\epsilon(T)^{c}\big) \leq \sum_{t=g(T)}^{T} C_5(\epsilon)e^{-C_4(\epsilon) 4\sqrt{t}}\leq C_5(\epsilon) T e^{-C_4(\epsilon) 4\sqrt{g(T)}}\leq C_5(\epsilon) T e^{-C_4(\epsilon) T^{1/8}}\,.
\]
 \end{proof}
 \begin{proof}[Proof of Equation~\ref{eq:get_ride_of_gT}]
 Using the triangular inequality we have 
\begin{align*}
    \left|w(T)- \frac{1}{\tT} \sum_{t=g(T)}^T\tw(t)\right|_{\infty}\leq  \left|w(T)- \frac{1}{T} \sum_{t=1}^T\tw(t)\right|_{\infty}+  \left|\frac{1}{T} \sum_{t=1}^T\tw(t)- \frac{1}{\tT} \sum_{t=g(T)}^T\tw(t)\right|_{\infty}\,.
\end{align*}
 It remains to notice that
  \begin{align*}
\left|\frac{1}{T} \sum_{t=1}^T\tw(t) -\frac{1}{\tT} \sum_{t=g(T)
}^T\tw(t) \right|_{\infty}&\leq \left|\frac{1}{T} \sum_{t=1}^T\tw(t) -\frac{1}{T} \sum_{t=g(T)
}^T\tw(t) \right|_{\infty}+ \left|\frac{1}{T} \sum_{t=g(T)}^T\tw(t) -\frac{1}{\tT} \sum_{t=g(T)
}^T\tw(t) \right|_{\infty}\\
&\leq \frac{g(T)}{T}+ \left(\frac{1}{\tT}-\frac{1}{T}\right)\tT\leq 2\frac{g(T)}{T}\\
&\leq \frac{2}{\sqrt{T}}\,,
\end{align*}
where in the last line we used $g(T)\leq \sqrt{T}$, by definition.
 \end{proof}
\subsection{Proof of Proposition~\ref{prop:regularity}}
The fact that there exists $\kappa<\kappa_0$ such that $\cB_{\infty}(\mu,\kappa)\subset\cS_{i(\mu)}$ is just a consequence of the fact that $\cS_{i(\mu)}$ is open. For such $\kappa$ we know that for any $\mu'\in\cB_{\infty}(\mu,\kappa)$ 
\[F(w,\mu')=\min_{i\neq i(\mu)}\inf_{\lambda\in\cS_i} \sum_{a=1}^K w_a\d(\mu_a',\lambda_a)\,.\]
Then thanks to Theorem~4 of \citet{degenne2019pure} (1.), for all $i$, the functions
\[
(w,\mu')\to \inf_{\lambda\in\cS_i} \sum_{a=1}^K w_a\d(\mu_a',\lambda_a)\,,
\]
are continuous on $\Sigma_K\times\Bar{\cB}_{\infty}(\mu,\kappa/2)$ (where $\Bar{B}$ denotes the closure of the set $B$), thus $F$ is continuous and then uniformly continuous on this compact set. Thus for all $\epsilon>0$ the exists $\kappa_\epsilon\leq \kappa/2$ such that 
\[
|\mu'-\mu''|_{\infty}\leq \kappa_\epsilon \Rightarrow |F(w,\mu')-F(w,\mu'')|\leq \epsilon\,.
\]

\subsection{Counter example for Assumption~\ref{assp:bounded_gradient}}
\label{app:counter_example}
We now present an example of problem where the sub-gradients can be unbounded. We set $\cM = \R^2$, $\cI=[1,2]$ and 
\[
\cS_1=\cB_\infty\big( (0,0), 1/4\big) \qquad \cS_2 = \{ (x,y)\in\cM:\ x>0,\,y>1/x\}\,.
\]
For the bandit problem $\mu = (0,0)$ we have $i(\mu)=1$ and 
\[
F(w,\mu) = \frac{1}{2} \frac{w_2^2}{w_1} + \frac{1}{2} \frac{w_1^2}{w_2}\,.
\]
Thus the gradient of $F(\cdot,\mu)$ at $w$ in the interior of the simplex is 
\[
\nabla F(w,\mu)=
\begin{bmatrix}
\frac{w_1}{w_2}-\frac{1}{2}\frac{w_2^2}{w_1^2}\\
\frac{w_2}{w_1}-\frac{1}{2}\frac{w_1^2}{w_2^2}\end{bmatrix}\!,
\]
which is unbounded when for example $w_1$ goes to 0 and $w_2$ is fixed.
\end{document}